\documentclass{article}

\PassOptionsToPackage{numbers, compress}{natbib}

\usepackage[preprint]{neurips_2025}

\usepackage[utf8]{inputenc} 
\usepackage[T1]{fontenc}    
\usepackage{hyperref}       
\usepackage{url}            
\usepackage{booktabs}       
\usepackage{amsfonts}       
\usepackage{nicefrac}       
\usepackage{microtype}      
\usepackage[table]{xcolor}         

\usepackage{wrapfig}
\usepackage{subcaption}
\usepackage{siunitx}
\usepackage[inline]{enumitem}
\usepackage{abaisero}
\usepackage[capitalise]{cleveref}
\usepackage[textsize=tiny]{todonotes}

\newcommand\jq{\joint{q}}
\newcommand\jv{\joint{v}}
\newcommand\ju{\joint{u}}

\newcommand\qigm{Q_{\operatorname{IGM}}}
\newcommand\vigm{V_{\operatorname{IGM}}}
\newcommand\aigm{A_{\operatorname{IGM}}}

\newcommand\qfixee{\qmodel_{\operatorname{fixee}}}
\newcommand\vfixee{\vmodel_{\operatorname{fixee}}}
\newcommand\afixee{\amodel_{\operatorname{fixee}}}

\newcommand\qfixsum{\qmodel_{\operatorname{FIX-sum}}}

\newcommand\qfixlin{\qmodel_{\operatorname{FIX-lin}}}

\newcommand\qfixmono{\qmodel_{\operatorname{FIX-mono}}}

\newcommand\aqfix{\qmodel_{\operatorname{+FIX}}}

\newcommand\intervention{\Delta}

\newcommand\aqfixsum{\qmodel_{\operatorname{+FIX-sum}}}

\newcommand\aqfixmono{\qmodel_{\operatorname{+FIX-mono}}}

\newcommand\aqfixlin{\qmodel_{\operatorname{+FIX-lin}}}

\newcommand\fc{\mathcal{FC}}
\newcommand\fcigm{\mathcal{FC}_{\operatorname{IGM}}}

\newcommand\fori{_{i\in\iset}}
\newcommand\qloss{\loss_{\qmodel}}

\newcommand\Protoss{\texttt{Protoss}}
\newcommand\Zerg{\texttt{Zerg}}
\newcommand\Terran{\texttt{Terran}}

\newcommand\sizev{\texttt{5vs5}}
\newcommand\sizex{\texttt{10vs10}}
\newcommand\sizexx{\texttt{20vs20}}

\newcommand\Protossv{\texttt{Protoss-5vs5}}
\newcommand\Zergv{\texttt{Zerg-5vs5}}
\newcommand\Terranv{\texttt{Terran-5vs5}}

\newcommand\Protossx{\texttt{Protoss-10vs10}}
\newcommand\Zergx{\texttt{Zerg-10vs10}}
\newcommand\Terranx{\texttt{Terran-10vs10}}

\newcommand\Protossxx{\texttt{Protoss-20vs20}}
\newcommand\Zergxx{\texttt{Zerg-20vs20}}
\newcommand\Terranxx{\texttt{Terran-20vs20}}

\renewcommand\Protossv{\texttt{P5}}
\renewcommand\Zergv{\texttt{Z5}}
\renewcommand\Terranv{\texttt{T5}}

\renewcommand\Protossx{\texttt{P10}}
\renewcommand\Zergx{\texttt{Z10}}
\renewcommand\Terranx{\texttt{T10}}

\renewcommand\Protossxx{\texttt{P20}}
\renewcommand\Zergxx{\texttt{Z20}}
\renewcommand\Terranxx{\texttt{T20}}

\newcommand\CrampedRoom{\texttt{Cramped-Room}}
\newcommand\AsymmetricAdvantages{\texttt{Asymmetric-Advantages}}
\newcommand\CoordinationRing{\texttt{Coordination-Ring}}
\newcommand\ForcedCoordination{\texttt{Forced-Coordination}}
\newcommand\CounterCircuit{\texttt{Counter-Circuit}}

\newcommand\pymarlone{\texttt{Pymarl}}
\newcommand\pymarl{\texttt{Pymarl2}}
\newcommand\jaxmarl{\texttt{JaxMARL}}

\title{Fixing Incomplete Value Function Decomposition \\for Multi-Agent Reinforcement Learning}

\author{%
  Andrea Baisero \\
  Khoury College of Computer Sciences \\
  Northeastern University \\
  Boston, MA 02115 \\
  \texttt{baisero.a@northeastern.edu} \\
  \And
  Rupali Bhati \\
  Khoury College of Computer Sciences \\
  Northeastern University \\
  Boston, MA 02115 \\
  \texttt{bhati.r@northeastern.edu} \\
  \And
  Shuo Liu \\
  Khoury College of Computer Sciences \\
  Northeastern University \\
  Boston, MA 02115 \\
  \texttt{liu.shuo2@northeastern.edu} \\
  \And
  Aathira Pillai \\
  Khoury College of Computer Sciences \\
  Northeastern University \\
  Boston, MA 02115 \\
  \texttt{pillai.aat@northeastern.edu} \\
  \And
  Christopher Amato \\
  Khoury College of Computer Sciences \\
  Northeastern University \\
  Boston, MA 02115 \\
  \texttt{c.amato@northeastern.edu} \\
}

\begin{document}

\maketitle

\begin{abstract}
Value function decomposition methods for cooperative multi-agent reinforcement learning compose joint values from individual per-agent utilities, and train them using a joint objective.
To ensure that the action selection process between individual utilities and joint values remains consistent, it is imperative for the composition to satisfy the \emph{individual-global max} (IGM) property.
Although satisfying IGM itself is straightforward, most existing methods (e.g., VDN, QMIX) have limited representation capabilities and are unable to represent the full class of IGM values, and the one exception that has no such limitation (QPLEX) is unnecessarily complex.
In this work, we present a simple formulation of the full class of IGM values that naturally leads to the derivation of QFIX, a novel family of value function decomposition models that expand the representation capabilities of prior models by means of a thin ``fixing'' layer.
We derive multiple variants of QFIX, and implement three variants in two well-known multi-agent frameworks.
%
We perform an empirical evaluation on multiple SMACv2 and Overcooked environments, which confirms that QFIX
\begin{enumerate*}[label=(\roman*)]
    \item succeeds in enhancing the performance of prior methods,
    \item learns more stably and performs better than its main competitor QPLEX, and
    \item achieves this while employing the simplest and smallest mixing models.
\end{enumerate*}
\end{abstract}

\section{Introduction}
\label{sec:introduction}

Centralized training for decentralized execution (CTDE)~\citep{lowe_multi-agent_2017,rashid_monotonic_2020,wang_qplex_2020} is a powerful framework for cooperative multi-agent reinforcement learning (MARL).
CTDE is characterized by a centralized training phase where privileged information is shared freely and used holistically to train the agents, and a decentralized execution phase where agents act independently in adherence to the standard constraints of decentralized control.
As a consequence of a training phase that is informed by the full team's behavior and experiences (and, when feasible, the environment state), CTDE is commonly associated with increased coordination between agents and superior performances.

Value function decomposition~\citep{sunehag_value-decomposition_2017,rashid_monotonic_2020,wang_qplex_2020} is a class of CTDE methods that construct a joint team value from individual per-agent utilities that encode agent behaviors.
By training the joint value on a joint centralized objective, the individual utilities are also indirectly trained, resulting in decentralized agent policies that can be executed independently.
Since its inception, value function decomposition has become a topic of great interest in cooperative MARL, with significant research effort put in both practical algorithms~\citep{sunehag_value-decomposition_2017,son_qtran_2019,rashid_weighted_2020,rashid_monotonic_2020,wang_qplex_2020,marchesini_stateful_2024} and theoretical understanding~\citep{wang_towards_2021,marchesini_stateful_2024}.
\emph{Individual-global max} (IGM)~\citep{son_qtran_2019} has been identified as a key property that connects individual utilities and joint values, ensuring that their associated decision making processes remain consistent.

In this work, we advance both theory and practice of value function decomposition.
We formulate a novel simple formulation of IGM-complete value function decomposition.
Our formulation
\begin{enumerate*}[label=(\roman*)]
    \item correctly addresses general decentralized partially observable control (avoiding strong assumptions like full observability or centralized control), and
    \item highlights the core mechanism that characterizes the full IGM-complete function class.
\end{enumerate*}
In contrast, prior methods fail to satisfy at least one of these criteria (usually the first, which limits the expressive capabilities and performance of models).
We introduce QFIX, a novel family of value function decomposition methods inspired by our formulation of IGM-complete decomposition.
QFIX employs a simple ``fixing'' network to extend the representation capabilities of prior methods.
We derive two main specializations of QFIX called QFIX-sum and QFIX-mono, respectively obtained by ``fixing'' VDN~\citep{sunehag_value-decomposition_2017} and QMIX~\citep{rashid_monotonic_2020}.
To provide further insights into the core mechanisms that make value function decomposition so effective, we also derive QFIX-lin, a third variant that technically falls just outside of the QFIX family, but combines QFIX-sum with a core component of QPLEX.
Finally, we extend prior work on stateful value function decomposition to QFIX.
Empirical evaluations on the StarCraft Multi-Agent Challenge v2 (SMACv2)~\citep{ellis_smacv2_2023} and Overcooked~\citep{carroll_utility_2020} demonstrates that QFIX
\begin{enumerate*}[label=(\roman*)]
    \item is effective at enhancing prior non-IGM-complete methods like VDN and QMIX,
    \item is simpler to implement and understand, and require smaller models than QPLEX, a state-of-the-art method in IGM-complete value function decomposition,
    \item is competitive or outperforms QPLEX while also showing more stable convergence.
\end{enumerate*}
An additional ablation study on model size confirms that the superior performance of QFIX is attributed to the intrinsic mixing approach rather than by augmenting baseline parameters.

\section{Related work}
\label{sec:related-work}

Value Decomposition Networks (VDN)~\citep{sunehag_value-decomposition_2017} are a precursor to value decomposition methods that employ a simple additive composition of individual utilities.
QMIX~\citep{rashid_monotonic_2020} employs a monotonic composition that generalizes the function class of VDN resulting in significant performance improvements.
Both VDN and QMIX have restricted function classes, and several methods have attempted to overcome the limits of purely additive or monotonic composition and achieve broader expressiveness.
Weighted-QMIX (WQMIX)~\citep{rashid_weighted_2020} aims to expand the function class of QMIX to non-monotonic cases so as to include optimal values $Q\opt$.
However, WQMIX is specifically developed for \emph{fully observable} multi-agent environments (MMDPs), and its theory does not generalize to \emph{partially observable} DecPOMDPs.
In contrast, QFIX is fully consistent with the general case of partially observable decentralized control.
\citet{son_qtran_2019} identify \emph{individual-global max} (IGM) as a core property that corresponds to consistency between the individual and joint decision making processes.
Notably, VDN and QMIX satisfy IGM, but are unable to represent the entire IGM-complete function class.
QTRAN~\citep{son_qtran_2019} identifies a set of constraints that are sufficient to imply IGM, and employs auxiliary objectives that softly enforce those constraints.
\citet{son_qtran_2019} argue that their constraints are also necessary for IGM under affine transformations, however they only show that one such affine transformation exists, rather than IGM being satisfied for all affine transformations.
In contrast, QFIX is both sufficient and necessary to imply IGM, thus directly achieving the full IGM-complete function class.
QPLEX~\citep{wang_qplex_2020} employs a dueling network decomposition and multiple layers of transformations to achieve the IGM-complete function class.
However, QPLEX employs complex transformations that are superfluous in relation to its representation capabilities, and falls short of identifying the core underlying mechanism that is singularly responsible to achieve the IGM function class.
In contrast, QFIX is both simpler to understand and to implement, and achieves the IGM function class with fewer smaller models.
QPLEX is one instance in the space of IGM-complete models, and our work opens a path to explore other instances that can further improve performance while satisfying IGM.

\section{Background}
\label{sec:background}

\subsection{Decentralized multi-agent control}
\label{sec:dec-pomdps}

A decentralized POMDP (Dec-POMDP)~\citep{oliehoek_concise_2016} generalizes single-agent partially observable control by accounting for multiple decentralized agents acting concurrently to solve a shared cooperative task.
A Dec-POMDP is defined by a tuple $\langle N, \sset, \left\{ \aset_1, \ldots, \aset_N \right\}, \left\{ \oset_1, \ldots, \oset_N \right\}, p, T, R, O, \gamma \rangle$ composed of:
\begin{enumerate*}[label=(\roman*)]
  \item number of agents $N \ge 2$;
  \item state space $\sset$;
  \item individual action and observation spaces $\aset_i$ and $\oset_i$;
  \item starting state distribution $p \in \Delta\sset$;
  \item state transition function $T \colon\sset\times\jaset \to\Delta\sset$;
  \item joint observation function $O \colon\jaset\times\sset \to\Delta\joset$;
  \item joint reward function $R \colon\sset\times\jaset \to\realset$; and
  \item discount factor $\gamma\in[0, 1)$.
\end{enumerate*}
The number of agents $N$ induces a set of agent indices $\iset \doteq [N]$.
Agent behaviors are generally modeled as stochastic policies $\policy_i \colon\hset_i \to\Delta\aset_i$ that act based on their respective history $h_i\in\hset_i \doteq \oset_i \times\left(\aset_i\times\oset_i \right)\kstar$.
Joint action, observation, and history spaces are defined as the respective Cartesian products $\jaset \doteq \bigtimes_i \aset_i$, $\joset \doteq \bigtimes_i \oset_i$, and $\jhset \doteq \bigtimes_i \hset_i$.
Therefore, joint actions $\ja = \left( a_1, \ldots, a_N \right)$, observations $\jo = \left( o_1, \ldots, o_N \right)$, and histories $\jh = \left( h_1, \ldots, h_N \right)$ are tuples of the respective individual actions, observations, and histories.
The combined behavior of all policies is represented as a joint (but still decentralized) policy $\jpolicy(\jh, \ja) \doteq \prod_i \policy_i(h_i, a_i)$ that factorizes accordingly.
Decentralized multi-agent control aims to find independent policies that jointly maximize the expected sum of discounted rewards $J^\jpolicy \doteq \Exp\left[ \sum_t \gamma^t R(s_t, \ja_t) \right]$.


We focus on approaches that model policies implicitly via parametric utilities $\qmodel_i \colon\hset_i\times\aset_i \to\realset$, typically via ($\epsilon$-)greedy action selection.
Individual utilities can be decomposed into corresponding values $\vmodel_i(h_i) \doteq \max_{a_i} \qmodel_i(h_i, a_i)$ and advantages $\amodel_i(h_i, a_i) \doteq \qmodel_i(h_i, a_i) - \vmodel_i(h_i)$.
When convenient, we employ shorthand notation for individual values $q_i \doteq \qmodel_i(h_i, a_i)$, $v_i \doteq \vmodel_i(h_i)$, and $u_i \doteq \amodel_i(h_i, a_i)$, and their joint tuples $\jq \doteq (q_1, \ldots, q_N)$, $\jv \doteq (v_1, \ldots, v_N)$, and $\ju \doteq (u_1, \ldots, u_N)$.

\subsection{Value function decomposition}
\label{sec:value-function-decomposition}

Value function decomposition methods~\citep{sunehag_value-decomposition_2017,rashid_monotonic_2020,wang_qplex_2020} construct joint values $\qmodel(\jh, \ja)$ from individual per-agent \emph{utilities} $\qmodel_i(h_i, a_i)$.
We specifically use the term \emph{utility} to underscore the fact that $\qmodel_i(h_i) \in \realset^{\aset_i}$ merely represents an ordering over actions, rather than any notion of expected performance.
Notably, $\qmodel_i$ is \emph{not} directly trained for policy evaluation or optimization, and neither $\qmodel_i(h_i, a_i) \approx \jqpolicy_i(h_i, a_i)$ nor $\qmodel_i(h_i, a_i) \approx Q\opt_i(h_i, a_i)$ are expected interpretations of well-trained utilities.

Value function decomposition methods employ joint models $\qmodel(\jh, \ja)$ that are a function of the individual utilities $\qmodel_i(h_i, a_i)$, and mainly differ in terms of the relationship that is enforced and the corresponding emergent properties.
The joint model $\qmodel(\jh, \ja)$ is trained on a \emph{joint} objective
%
that optimizes the joint values and behavior, and indirectly trains the individual utilities and behaviors,
\begin{equation}
  \label{eq:value-decomposition:loss}
  \qloss(\jh, \ja, r, \jo) \doteq \frac{1}{2} \left( r + \gamma \max_{\ja'} \qmodel^-(\jh\ja\jo, \ja') - \qmodel(\jh, \ja) \right)^2 \,.
\end{equation}

\paragraph{Individual-global max}

\citet{son_qtran_2019} identify individual-global max (IGM) as a useful property of decomposition models to achieve decentralized action selection and address scaling concerns.

\begin{definition}[Individual-Global Max]
  \label{thm:igm}
  %
  Individual utilities $\{ Q_i(h_i, a_i) \}\fori$ and joint values $Q(\jh, \ja)$ satisfy \emph{individual-global max} (IGM) iff $\bigtimes_i \argmax_{a_i} Q_i(h_i, a_i) = \argmax_\ja Q(\jh, \ja)$.\footnote{We employ set notation and Cartesian products to highlight that maximal actions may not be unique.}
\end{definition}

IGM denotes whether the individual and global decision making processes are equivalent, and reduces the complexity of finding the maximal joint action from exponential to linear in the number of agents:
For a given joint history $\jh$, the full search over the joint action space $\jaset$ can be replaced with $N$ independent searches over the individual action spaces $\aset_i$.
VDN and QMIX are well-known models that satisfy IGM; however, their function classes do not span the full class of IGM values.

\begin{definition}[IGM Function Class]
  \label{thm:igm-completeness}
  We say a function class of individual utilities $\{ Q_i(h_i, a_i) \}\fori$ and joint values $Q(\jh, \ja)$ is IGM-complete if it contains all and only functions that satisfy IGM.
\end{definition}



%

\paragraph{VDN: additive decomposition}

Value Decomposition Network (VDN)~\citep{sunehag_value-decomposition_2017} is a precursor to value function decomposition that uses a non-parametric structure $\qvdn(\jh, \ja) \doteq \sum_i \qmodel_i(h_i, a_i)$.

\paragraph{QMIX: monotonic decomposition}

QMIX~\citep{rashid_monotonic_2020} constructs joint values as a \emph{monotonic} function of individual utilities, $\qmix(\jh, \ja) \doteq \fmono(q_1, \ldots, q_N)$, where $\fmono \colon\realset^N \to\realset$ is a parametric mixing network that satisfies monotonicity, $\partial_{q_i} \fmono \geq 0$.
The monotonic composition of QMIX strictly generalizes the additive composition of VDN, though it still falls short of the full IGM function class.

\paragraph{QPLEX: IGM-complete decomposition}

QPLEX~\citep{wang_qplex_2020} reframes IGM in terms of advantages, and employs dueling network decomposition to achieve full function class equivalence with IGM.
Given utilities $Q_i(h_i, a_i)$ and joint action-values $Q(\jh, \ja)$, corresponding values and advantages are defined,
\begin{align}
  V_i(h_i) &\doteq \max_{a_i} Q_i(h_i, a_i) \,, & A_i(h_i, a_i) &\doteq Q_i(h_i, a_i) - V_i(h_i) \,, \\
  V(\jh) &\doteq \max_\ja Q(\jh, \ja) \,, & A(\jh, \ja) &\doteq Q(\jh, \ja) - V(\jh) \,.
\end{align}
%
\citet{wang_qplex_2020} reformulate IGM as a set of constraints between individual and joint advantages.

\begin{proposition}[Advantage Constraints]
  \label{thm:igm:advantage-constraint}
  Individual utilities $\{ Q_i(h_i, a_i) \}\fori$ and joint values $Q(\jh, \ja)$ satisfy IGM iff, $\forall \jh \in\jhset$, $\forall \ja\opt \in\jaset\opt(\jh)$, and $\forall \ja \in\jaset\setminus\jaset\opt(\jh)$,
  \begin{align}
    A(\jh, \ja\opt) &= 0 \,, & A_i(h_i, a\opt_i) &= 0 \,, \\
    A(\jh, \ja) &< 0 \,, & A_i(h_i, a_i) &\leq 0 \,,
  \end{align}
  where $\jaset\opt(\jh) \doteq \argmax_\ja Q(\jh, \ja)$ is the set of maximal joint actions according to the joint values.
\end{proposition}

QPLEX employs a mixing structure that enforces \cref{thm:igm:advantage-constraint}.
Individual utilities $\qmodel_i(h_i, a_i)$ are decomposed into $\vmodel_i(h_i)$ and $\amodel_i(h_i, a_i)$ and transformed using centralized information,
\begin{align}
  \vmodel_i(\jh) &\doteq w_i(\jh) \vmodel_i(h_i) + b_i(\jh) \,, &
  \amodel_i(\jh, a_i) &\doteq w_i(\jh) \amodel_i(h_i, a_i) \,,
\end{align}
where $w_i \colon\jhset \to\realset_{>0}$ are parametric positive weights and $b_i \colon\jhset \to\realset$ are parametric biases.
These transformed values are aggregated as weighted sums,
\begin{align}
  \vplex(\jh) &\doteq \sum_i \vmodel_i(\jh) \,, &
  \label{eq:qplex:advantage}
  \aplex(\jh, \ja) &\doteq \sum_i \lambda_i(\jh, \ja) \amodel_i(\jh, a_i) \,, 
\end{align}
where $\lambda_i\colon \jhset\times\jaset \to\realset_{>0}$ are parametric positive weights.
Finally, the QPLEX joint values are obtained by recombining aggregate values and advantages, $\qplex(\jh, \ja) \doteq \vplex(\jh) + \aplex(\jh, \ja)$.

This sequence of decomposition, transformations, and recomposition, combined with positive weights $w_i$ and $\lambda_i$, results in the constraint from \cref{thm:igm:advantage-constraint} being satisfied.
Consequently, \citet{wang_qplex_2020} appeal to the universal approximation theorem (UAT) to argue that the function class of QPLEX is IGM-complete.
In \cref{sec:is-qplex-igm-complete}, we address technical concerns and conclude that, based on a \emph{weaker} form of UAT, the function class realizable by QPLEX is technically that of \emph{measurable} IGM values.

\paragraph{Stateful value function decomposition}

Practical implementations of value function decomposition methods often employ stateful joint values $Q(\jh, s, \ja)$ and diverge from the stateless theoretical derivations in ways that may undermine core IGM properties, e.g., as seen for QMIX in \pymarlone~\cite{rashid_monotonic_2020}, QMIX in \pymarl~\cite{ellis_smacv2_2023}, and both QMIX and QPLEX in \jaxmarl~\cite{rutherford_jaxmarl_2024}) 
To address the effects of state in value function decomposition, \citet{marchesini_stateful_2024} formulate a state-compliant version of IGM.

\begin{definition}[Stateful-IGM]
\label{thm:stateful-igm}

  Individual utilities $\{ Q_i(h_i, a_i) \}\fori$ and stateful joint values $Q(\jh, s, \ja)$ satisfy stateful-IGM iff $\bigtimes_i \argmax_{a_i} Q_i(h_i, a_i) = \argmax_{\ja} \Exp_{s\mid \jh}\left[ Q(\jh, s, \ja) \right]$.
  %
\end{definition}

\citet{marchesini_stateful_2024} show that the stateful implementations of QMIX and QPLEX continue to satisfy IGM, while the stateful implementation of QPLEX (which employs historyless stateful weights $w_i(s), \lambda_i(s, \ja)$) fails to achieve the full IGM function class.
Nonetheless, stateful implementations often perform well in practice, and remain a common occurrence.

\section{Fixing incomplete value function decomposition}

Although QPLEX is IGM-complete, it is expressed as a convoluted sequence of transformations that are never fully motivated or justified.
Fully unrolling the QPLEX values in terms of the individual utilities, we get $\qplex(\jh, \ja) = \sum_i w_i(\jh) \vmodel_i(h_i) + b_i(\jh) + w_i(\jh) \lambda_i(\jh, \ja) \amodel_i(h_i, a_i)$,
%
%
a complex expression that raises questions about which components are truly important or necessary, e.g., the product of individual advantages with two types of positive weights $w_i(\jh)$ and $\lambda_i(\jh, \ja)$ appears to be redundant.
Ultimately, QPLEX only represents one instance in the space of all IGM-complete models, and whether simpler or better-performing models exist remains an open question.

The convoluted nature of the QPLEX transformations motivate us to find a simpler and more general formulation of IGM-complete decomposition.
In this section, we first present a simple formulation of the IGM-complete function class.
Then, we use this formulation to derive QFIX, a novel family of value function decomposition models that operate by ``fixing'' (read: expanding) the representation capabilities of prior non-IGM-complete models.
We derive two primary instances of QFIX based on ``fixing'' VDN and QMIX respectively, and a third instance designed to resemble QPLEX.
Then, we derive \emph{additive} QFIX (Q+FIX), a simple variant of QFIX that achieves significant practical performance gains, and derive Q+FIX counterparts of the QFIX instances.
Finally, we discuss stateful variants of QFIX and how the use of centralized state information affects its theoretical properties.

\subsection{A simple parameterization of the IGM function class}
\label{sec:qigm}

We aim to formalize IGM-complete value function decomposition in its simplest and most essential form.
We begin by simplifying \cref{thm:igm:advantage-constraint}, noting that three of its four constraints are satisfied by definition; the only one that requires active enforcement is $A_i(h_i, a\opt_i) = 0$, or equivalently $A(\jh, \ja) = 0 \implies \forall i\, (A_i(h_i, a_i) = 0)$.
However, we also note that \cref{thm:igm:advantage-constraint} is actually underspecified, and misidentifies the case where $\forall i (A_i(h_i, a_i) = 0)$ and $A(\jh, \ja) < 0$ as compliant with IGM when it is not.\footnote{Luckily, this issue is exclusive to \cref{thm:igm:advantage-constraint}, and QPLEX itself does not suffer from the same issue.}
To address this case, we need $\forall i\, (A_i(h_i, a_i) = 0) \implies A(\jh, \ja) = 0$.

\begin{proposition}[Simplified Advantage Constraints]
  \label{thm:igm:advantage-constraint:simplified}
  Individual utilities $\{ Q_i(h_i, a_i) \}\fori$ and joint values $Q(\jh, \ja)$ satisfy IGM iff $\forall i\, (A_i(h_i, a_i) = 0) \iff A(\jh, \ja) = 0$, or equivalently, via contraposition, $\exists i\, (A_i(h_i, a_i) < 0) \iff A(\jh, \ja) < 0$.
\end{proposition}

In essence, constructing joint advantages that are negative iff any of the individual advantages are negative is both sufficient and necessary to satisfy IGM.
Consider the purposefully named function
\begin{equation}
  \label{eq:qigm}
  \qigm(\jh, \ja) \doteq w(\jh, \ja) f(u_1, \ldots, u_N) + b(\jh) \,,
\end{equation}
where $w \colon\jhset\times\jaset \to\realset_{>0}$ is an arbitrary positive function of joint history and joint action, $b \colon\jhset \to\realset$ is an arbitrary function of joint history, and $f \colon\realset_{\le 0}^N \to\realset_{\le 0}$ is a non-positive function that is zero iff all inputs are zero, e.g., $f(u_1, \ldots, u_N) = \sum_i u_i$ is a simple instance of $f$.
Then,
\begin{align}
  \vigm(\jh) &\doteq \max_{\ja} \qigm(\jh, \ja) = b(\jh) \,, \\
  \aigm(\jh, \ja) &\doteq \qigm(\jh, \ja) - \vigm(\jh) = w(\jh, \ja) f(u_1, \ldots, u_N) \,.
\end{align}

Essentially, $\qigm$ denotes a relationship where any deviation from individual maximality (characterized by at least one negative utility $u_i < 0$, and corresponding to a negative $f(u_1, \ldots, u_N) < 0$) is transformed into an arbitrary deviation $w(\jh, \ja) f(u_1, \ldots, u_N) < 0$ from joint maximality (and vice versa).
Per \cref{thm:igm:advantage-constraint:simplified}, $\qigm$ represents the full IGM function class.

\begin{proposition}
  \label{thm:qigm:igm}
  For any $f$, $w$, and $b$, values $\{ Q_i \}\fori$ and $\qigm$ satisfy IGM.
  For any $f$, and given free choice of $w$ and $b$, the function class of $\{ Q_i \}\fori$ and $\qigm$ is IGM-complete.
  (Proof in \cref{sec:proof:qigm:igm}.)
\end{proposition}

$\qigm$ is a simple formulation of the IGM function class based on a single weighted (via $w$) transformation (via $f$) of individual advantages.
Next, we explore how this formulation directly inspires the derivation of QFIX, a closely related novel family of value function decomposition models.

\subsection{QFIX}
\label{sec:qfix}

\begin{figure*}
  \centering
  
  \begin{subfigure}{0.49\linewidth}
    \centering
    \includegraphics[width=\linewidth]{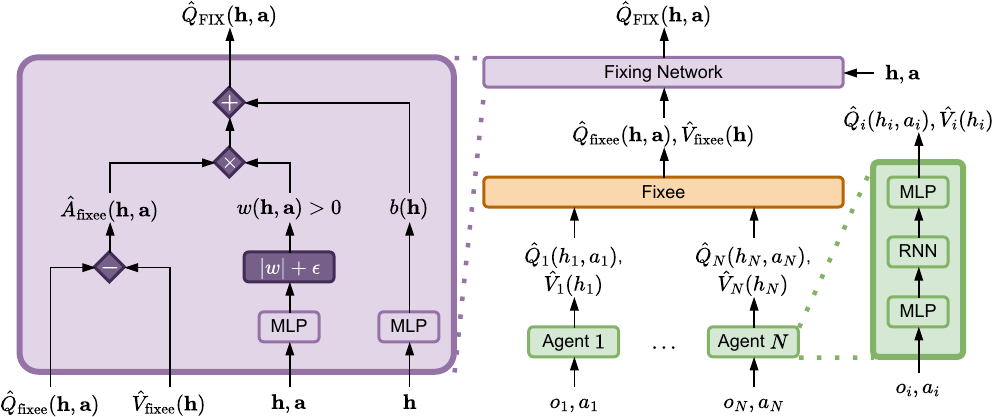}
  \end{subfigure}
  \hfill
  \begin{subfigure}{0.49\linewidth}
    \centering
    \includegraphics[width=\linewidth]{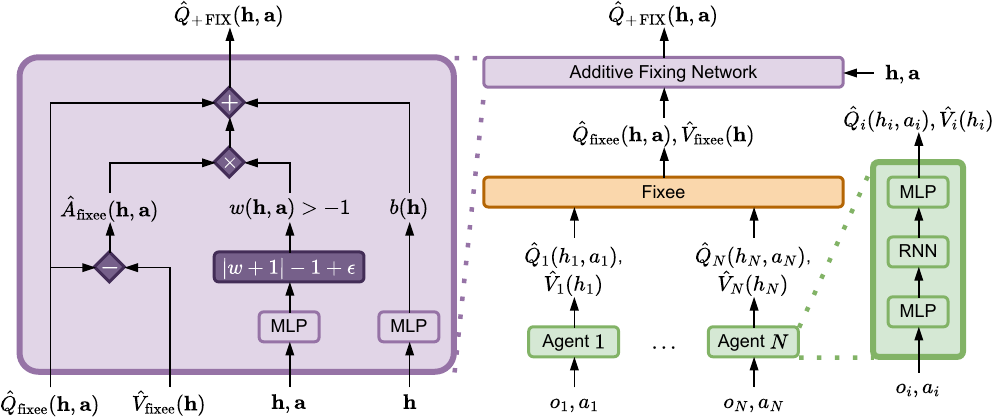}
  \end{subfigure}
  \caption{Diagrams for QFIX (left) and Q+FIX (right).}
  \label{fig:qfix:diagrams}
\end{figure*}

Let $\qfixee(\jh, \ja)$ denote a ``fixee'' value function decomposition model that satisfies IGM but is not IGM-complete, e.g., VDN or QMIX.
\Cref{eq:qigm} suggests a method to ``fix'' $\qfixee$ and expand its function class to match the full class of IGM functions.
We can extend the expressiveness of $\qfixee$ by processing it through a ``fixing'' network that resembles \cref{eq:qigm},
\begin{equation}
  \label{eq:qfix}
  \qfix(\jh, \ja) \doteq w(\jh, \ja) \afixee(\jh, \ja) + b(\jh) \,,
\end{equation}
where $w \colon\jhset\times\jaset \to\realset_{>0}$ is a parametric positive model, $b \colon\jhset \to\realset$ is a parametric model, and $\afixee \colon\jhset\times\jaset \to\realset_{\le 0}$ is the non-positive joint advantage of the fixee as defined by
\begin{align}
  \vfixee(\jh) &\doteq \max_\ja \qfixee(\jh, \ja) \,, &
  \afixee(\jh, \ja) &\doteq \qfixee(\jh, \ja) - \vfixee(\jh) \,.
\end{align}

See \cref{fig:qfix:diagrams} for a diagram of QFIX.
We note that $\afixee(\jh, \ja) = 0$ iff the joint action $\ja$ is maximal according to $\qfixee$, and negative otherwise.
Given that $\qfixee$ satisfies IGM by assumption, $\ja$ is maximal iff the individual actions $a_i$ are maximal according to $\qmodel_i(h_i, a_i)$, or, equivalently, iff $\amodel_i(h_i, a_i) = 0$.
In short, $\afixee(\jh, \ja)$ satisfies the requirements of $f$ under \cref{eq:qigm}.

\begin{proposition}
  \label{thm:qfix}
  QFIX satisfies IGM.
  The function class of QFIX is that of (measurable) IGM values.
  %
  (Proof in \cref{sec:proof:qfix}.)
\end{proposition}

Given the free choice of fixee model $\qfixee$, QFIX really represents a wide family of value function decomposition models.
This allows us to consider more or less complex fixees (e.g., VDN, QMIX) and explore various possible tradeoffs between minimizing the complexity of the fixee model and minimizing the ``fixing'' burden on the fixing models $w, b$.
In our empirical evaluation, we will generally find that the fixing burden on $w, b$ is not significant, and that it is perfectly reasonable to combine QFIX with simpler fixees like VDN or \emph{tiny} versions of parametric fixees like QMIX.


\paragraph{Relationship to QPLEX}
\label{sec:relationship-qplex}

The advantage component of QFIX, $w(\jh, \ja) \afixee(\jh, \ja)$, is similar to one of the transformations of QPLEX, $\sum_i \lambda_i(\jh, \ja) \amodel_i(\jh, a_i)$, which comparably applies positive weights to transformed aggregates of the individual advantages.
This similarity is no coincidence, as it is specifically that component of QPLEX that is singularly responsible for ensuring IGM-completeness; it is a more convoluted form of our proposed fixing structure.
However, QPLEX also employs various other transformations that do not contribute to the IGM-complete function class, and their necessity remains questionable (beyond general considerations of modeling structure and size).

The weights $\lambda_i(\jh, \ja)$ employed by QPLEX are also more complex in that there is one such model per agent, and each is implemented via self-importance.
In contrast, we employ a simpler structure based on a single model implemented as a simple feed-forward network, and still manage to achieve performance improvements.
Our formulation is simpler in that it focuses entirely on this single transformation, which is minimally sufficient to guarantee IGM-completeness.

\paragraph{Fixing VDN}
\label{sec:qfix-sum}

We define QFIX-sum as an instance of QFIX based on ``fixing'' VDN, i.e., with $\qfixee(\jh, \ja) = \qvdn(\jh, \ja)$, which results in (see \cref{sec:derivations:qfix-sum} for an explicit derivation)
\begin{equation}
  \qfixsum(\jh, \ja) = w(\jh, \ja) \smash{\sum_i} \amodel_i(h_i, a_i) + b(\jh) \,.
\end{equation}

\paragraph{Fixing QMIX}
\label{sec:qfix-mono}

We define QFIX-mono as an instance of QFIX based on ``fixing'' QMIX, i.e., with $\qfixee(\jh, \ja) = \qmix(\jh, \ja)$, which results in (see \cref{sec:derivations:qfix-mono} for an explicit derivation)
\begin{equation}
  \qfixmono(\jh, \ja) = w(\jh, \ja) \left( \fmono(q_1, \ldots, q_N) - \fmono(v_1, \ldots, v_N) \right) + b(\jh) \,.
\end{equation}

\paragraph{Simplifying QPLEX}

Given the discussed similarity between QFIX and QPLEX, we may consider another variant of QFIX that also applies per-agent positive weights $w_i(\jh, \ja) > 0$.
Due to the linear structure that generalizes the additive structure of QFIX-sum, we call this variant QFIX-lin.
\begin{equation}
    \label{eq:qfixlin}
    \qfixlin(\jh, \ja) \doteq \smash{\sum_i} w_i(\jh, \ja) \amodel_i(h_i, a_i) + b(\jh) \,.
\end{equation}

QFIX-lin does not strictly satisfy the form of \cref{eq:qfix}, however, it represents a close enough variant of QFIX-sum that we consider it QFIX-adjacent and name it accordingly.
QFIX-lin is a strict generalization of QFIX-sum, which can be recovered as a special case where all the weights $w_i(\jh, \ja)$ are equal.
Formally, we must prove the IGM properties of QFIX-lin separately.

\begin{proposition}
    \label{thm:qfix-lin}
    QFIX-lin satisfies IGM.
    The function class of QFIX-lin is that of (measurable) IGM values.
    %
    (Proof in \cref{sec:proof:qfix-lin}.)
\end{proposition}

\paragraph{Recovering the fixee model}
\label{sec:qfix-recovering}

We note that QFIX is able recover the fixee model via $w(\jh, \ja) = 1$ and $b(\jh) = \vfixee(\jh)$, for which $\qfix(\jh, \ja) = \afixee(\jh, \ja) + \vfixee(\jh) = \qfixee(\jh, \ja)$.
Such values of $w(\jh, \ja)$ and $b(\jh)$ establish a direct relationship between the fixee and fixed models, which is relevant as we next use this relationship to derive a better-performing \emph{additive} variant of QFIX.

\subsection{Additive QFIX (Q+FIX)}
\label{sec:q+fix}

In this section, we further derive a simple reparameterization of QFIX which, albeit having the same theoretical properties, achieves significant practical performance improvements.
This variant takes on an additive form when compared to the fixee model, hence its name \emph{additive QFIX} (Q+FIX).

As previously noted, the values of $w(\jh, \ja) = 1$ and $b(\jh) = \vfixee(\jh)$ hold a special significance for QFIX.
Q+FIX is obtained by reparameterizing $w$ and $b$ to incorporate such values additively,
\begin{align}
  \aqfix(\jh, \ja) &\doteq (w(\jh, \ja) + 1) \afixee(\jh, \ja) + ( b(\jh) + \vfixee(\jh) ) \nonumber \\
  &= \qfixee(\jh, \ja) + w(\jh, \ja) \afixee(\jh, \ja) + b(\jh) \,,
\end{align}
where $w \colon\jhset\times\jaset \to\realset_{>-1}$ is a parametric model constrained by $w(\jh, \ja) > -1$, $b \colon\jhset \to\realset$ is a parametric model, and $\qfixee$ and $\afixee$ are the fixee action-values and advantages.
%
Note that, with the reparameterization of $w$, its constraint has changed; Since $w(\jh, \ja) + 1 > 0$ must satisfy the positivity constraint from QFIX, the corresponding constraint for Q+FIX is therefore $w(\jh, \ja) > -1$.

See \cref{fig:qfix:diagrams} for a diagram.
This reparameterization allows Q+FIX to more directly exploit the original form of the fixee model, extending its representation via a separate additive component $\intervention(\jh, \ja) \doteq w(\jh, \ja) \afixee(\jh, \ja) + b(\jh)$ we call the \emph{fixing intervention}.
Because Q+FIX is a simple reparameterization of QFIX, the results from \cref{thm:qfix,thm:qfix-lin} apply trivially to their Q+FIX counterparts.
Next, we look at specific instances and other relevant implementation details.

\paragraph{Q+FIX-\{sum,mono,lin\}}

The Q+FIX counterparts to QFIX-\{sum,mono,lin\} are as follows.
See \cref{sec:derivations:q+fix-sum,sec:derivations:q+fix-mono,sec:derivations:q+fix-lin} for their corresponding derivations and specialized diagrams.
\begin{align}
  \aqfixsum(\jh, \ja) &= \sum_i \qmodel_i(h_i, a_i) + w(\jh, \ja) \sum_i \amodel_i(h_i, a_i) + b(\jh) \,, \\
  \aqfixmono(\jh, \ja) &= \fmono( \joint{q} ) + w(\jh, \ja) \left( \fmono( \joint{q} ) - \fmono( \joint{v} ) \right) + b(\jh) \,, \\
  \aqfixlin(\jh, \ja) &= \sum_i \qmodel_i(h_i, a_i) + \sum_i w_i(\jh, \ja) \amodel_i(h_i, a_i) + b(\jh) \,.
\end{align}

\paragraph{Detaching the advantages}
\label{sec:detach}

The additive form of Q+FIX enables the use of an implementation detail already employed by QPLEX that significantly improves performance: the detachment of the advantages when computing gradients.
This can be expressed using the stop-gradient operator,\footnote{The stop-gradient function is a mathematical anomaly whose value behaves like the identity function, $\stopg\left[ x \right] = x$, while its gradient behaves like the zero function, $\nabla_x \stopg\left[ x \right] = 0$. It is a functionality commonly provided by deep learning frameworks, e.g., \texttt{pytorch} provides this via the \texttt{Tensor.detach()} method.}
\begin{equation}
  \label{eq:q+fix}
  \aqfix(\jh, \ja) = \qfixee(\jh, \ja) + w(\jh, \ja) \stopg[ \afixee(\jh, \ja) ] + b(\jh) \,.
\end{equation}

The reason why detaching the advantages improves performance is not fully understood.
\citet[Appendix~B.2]{wang_qplex_2020} argue that it (cit.) ``\emph{increases the optimization stability of the max operator of the dueling structure}'', in reference to dueling networks~\citep{wang_dueling_2016}.
However, the connection between the detach and dueling networks remains unclear.
Instead, we hypothesize that detaching the advantage may mitigate adverse effects that the fixing structure may have on the gradients $\nabla_{\theta_i} \aqfix(\jh, \ja)$ of the joint values w.r.t. the agent parameters $\theta_i$ (see \cref{sec:detach-hypothesis}).

\paragraph{Annealing the intervention}


Another implementation detail we found to be occasionally useful to stabilize learning has been to introduce the fixing intervention smoothly during the early stages of training ($\approx 5\%$ of total timesteps) by employing an auxiliary loss $\lambda_\Delta \cdot \intervention^2(\jh, \ja)$ that minimizes the squared intervention, with a weight $\lambda_\Delta$ that is annealed from a starting value down to $0$, ultimately disabling this intervention loss.
This likely ensures that the early stages of training are focused on bootstrapping the fixee values, so that the fixing intervention can focus on making smaller adjustments.

\subsection{Stateful variants}
\label{sec:stateful-qfix}

As with QMIX and QPLEX, we may consider stateful variants of QFIX that partially deviate from the stateless theory developed so far.
Such variants warrant an explicit discussion on the implications of employing centralized state information~\citep{marchesini_stateful_2024}.
Different versions of stateful QFIX are possible by combining stateless/stateful fixees with stateless/stateful fixing networks.
As Q+FIX is a simple reparameterization of QFIX, its properties w.r.t the use of state are the same.
We briefly summarize the conclusions for two main stateful variants of QFIX, which are comparable to those for stateful QPLEX~\citep{marchesini_stateful_2024}:
\begin{enumerate*}[label=(\roman*)]
  \item[(History-State QFIX)] When employing history-state fixing models $w(\jh, s, \ja)$ and $b(\jh, s)$, QFIX both satisfies IGM and achieves a form of IGM-complete function class.
  \item[(State-Only QFIX)] When employing state-only fixing models $w(s, \ja)$ and $b(s)$, QFIX continues to satisfy IGM, but fails to achieve the IGM-complete function class.
\end{enumerate*}
See additional discussion in \cref{sec:appendix:stateful-qfix}.

%
%
%

\begin{figure}
  \begin{subfigure}{\linewidth}
    \centering
    \includegraphics[width=\linewidth]{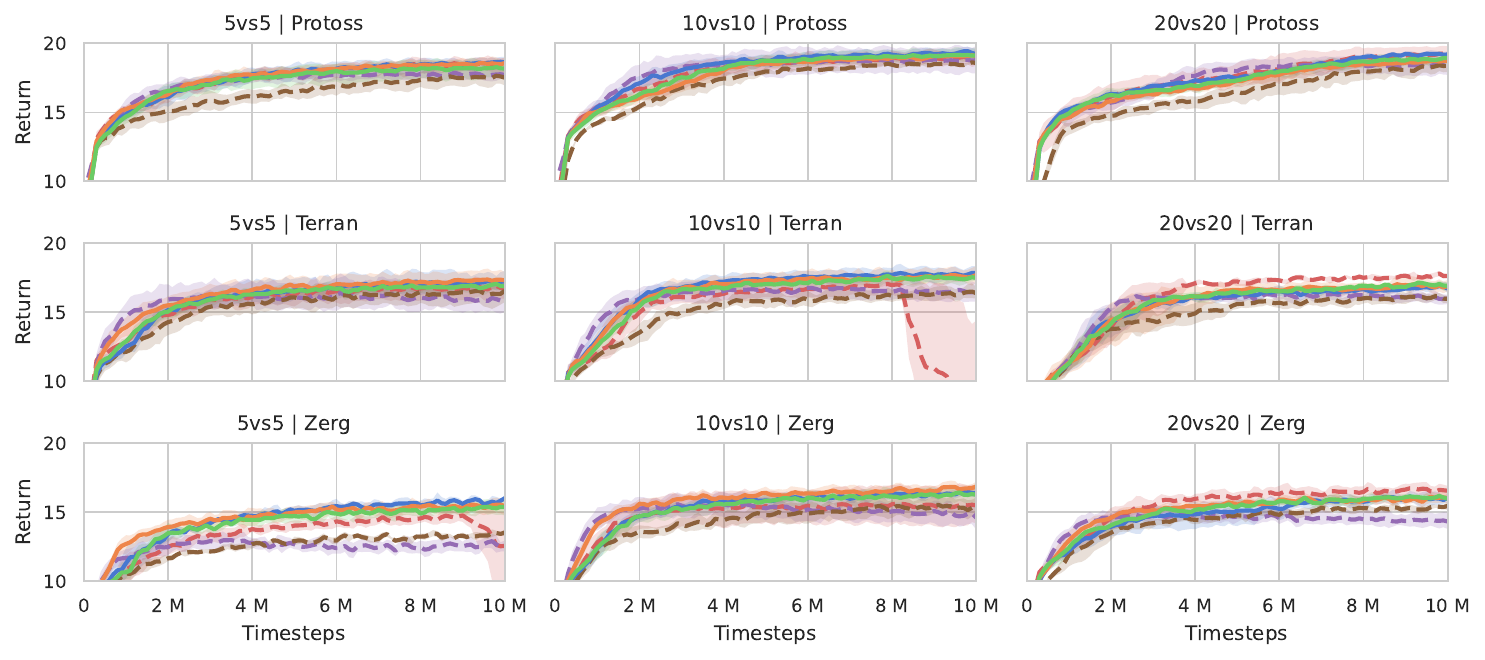}
    \caption{SMACv2 return mean ($5$ seeds).}
    \label{fig:results:smacv2:mean-returns}
  \end{subfigure}

  \begin{subfigure}{0.44\linewidth}
    \centering
    \includegraphics[height=2.4cm]{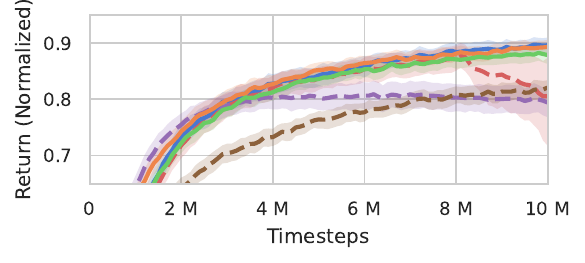}
    \caption{SMACv2 return mean aggregate ($45$ seeds).}
    \label{fig:results:smacv2:mean-returns:aggregate}
  \end{subfigure}
  \begin{subfigure}{0.54\linewidth}
    \centering
    \includegraphics[height=2.4cm]{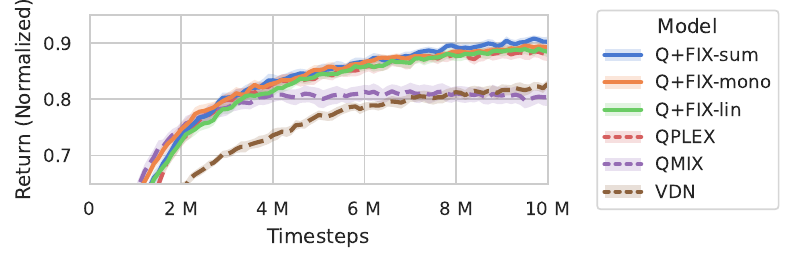}
    \caption{SMACv2 return IQM aggregate ($45$ seeds).}
    \label{fig:results:smacv2:iqm-returns:aggregate}
  \end{subfigure}

  \caption{SMACv2 results, bootstrapped $95\%$ CI.
    Aggregate returns are normalized per-task via $\tilde G_i \doteq \frac{G_i - \min_k G_k}{\max_k G_k - \min_k G_k}$, where $\{ G_i \}_i$ is the total set of returns logged by all models in a given task.}
  \label{fig:results:smacv2:returns}
\end{figure}

\vspace{-0.2cm}
\section{Evaluation}
\label{sec:evaluation}
\vspace{-0.2cm}

We perform an empirical evaluation of Q+FIX in
two popular multi-agent frameworks, \pymarl\ and \jaxmarl.
\cref{sec:architectures,sec:resources} contains practical details on architectures and used resources.

\paragraph{StarCraft Multi-Agent Challenge v2 (SMACv2)}

\pymarl\ 
provides baseline implementations for
SMACv2~\cite{ellis_smacv2_2023}, a popular benchmark for cooperative multi-agent control based on the real-time strategy game StarCraft II.
SMACv2 features two battling teams composed by configurable races, race-dependent and stochastically determined unit types, and team sizes.
Our empirical evaluation is based on $9$ scenarios obtained by combining the $3$ races (\Protoss, \Terran, and \Zerg) with $3$ team sizes (\sizev, \sizex, and \sizexx).
We use shorthand labels, e.g., \Protossv, \Terranx, \Zergxx.
\pymarl\ provides base implementations for VDN, QMIX, and QPLEX, and we implemented Q+FIX-\{sum,mono,lin\}.

%

\cref{fig:results:smacv2:mean-returns} contains the evaluation results based on mean performance, with $5$ independent runs per model per scenario.
As expected, VDN fails to be a competitive baseline on its own accord.
Fixing VDN via Q+FIX-sum, we are able to overcome this limitation, as noted by the corresponding performance gap.
QMIX sometimes exhibits fast initial learning speeds, albeit often to a sub-competitive final performance (\Protossv, \Terranv, \Terranx, \Zergx, \Terranxx, \Zergxx).
Fixing QMIX via Q+FIX-mono, we are often able to exploit the initial learning speed and complement it with improved convergence performance.
QPLEX is highly competitive and performs very well in some scenarios (\Protossv, \Protossxx, \Terranxx, \Zergxx), but underperforms in others (\Terranv, \Protossx, \Zergx), and exhibits troubling instabilities (\Zergv, \Terranx).
Q+FIX-lin, as the simplified variant inspired by QPLEX, manages to avoid such convergence instabilities, arguably as a consequence of the simpler structure.
Q+FIX-\{sum,mono,lin\} achieve similar performances in most cases.
Overall, Q+FIX-sum may be slightly outperforming other variants in some scenarios (\Terranv, \Zergv), possibly an indication that a simpler compositions are not just sufficient but possibly preferable.

In accordance to the methodology suggested by \citet{agarwal_deep_2021} to improve statistical significance and alleviate the impact of outliers,  \cref{fig:results:smacv2:mean-returns:aggregate,fig:results:smacv2:iqm-returns:aggregate} contain (normalized) aggregate results based on mean and interquantile mean (IQM).
%
Even ignoring the unstable convergence of QPLEX via the aggregate IQM results,
it is clear that the Q+FIX variants continue to outperform QPLEX at least marginally.
These results demonstrate that Q+FIX succeeds in enhancing the performance of its fixees, raising them to a level comparable to QPLEX while maintaining a more stable convergence.

\begin{wraptable}{r}{0.45\linewidth}
    \centering
    \caption{SMACv2 mixer sizes.}
    \label{tab:sizes}
    
    \tiny
    \bgroup 
    \setlength{\tabcolsep}{0.5em}
    \begin{tabular}{lrrrrrr}
        & \multicolumn{3}{c}{\Protoss}
        & \multicolumn{3}{c}{\Terran, \Zerg} \\
        & \sizev & \sizex & \sizexx
        & \sizev & \sizex & \sizexx \\
        \toprule
        QMIX
            & 38 k & 83 k & 201 k
            & 36 k & 79 k & 194 k \\
        QPLEX
            & 135 k & 326 k & 882 k
            & 126 k & 308 k & 846 k \\
        \rowcolor{gray!20}
        Q+FIX-sum
            & 20 k & 50 k & 138 k
            & 19 k & 48 k & 133 k \\
        Q+FIX-mono
            & 54 k & 180 k & 743 k
            & 50 k & 169 k & 708 k \\
        \rowcolor{gray!20}
        Q+FIX-lin
            & 21 k & 51 k & 140 k
            & 19 k & 48 k & 135 k \\
        \bottomrule
    \end{tabular}
    \egroup 
\end{wraptable}

\cref{tab:sizes} shows the sizes of \emph{mixing models} for the all methods that have one (smallest highlighted).
Notably, Q+FIX-\{sum,lin\} employ the smallest mixing models by a significant margin, indicating that their performance is a consequence of our proposed mixing structure over larger parameterizations.

\cref{sec:additional-results:smacv2} contains additional discussion on the SMACv2 evaluation, implementation details and chosen metrics, additional \emph{winrate} results, \emph{probability-of-improvement}~\cite{agarwal_deep_2021} results, and an ablation on model size for Q+FIX-mono and QMIX.

\vspace{-0.35cm}
\paragraph{Overcooked}

\jaxmarl~\cite{rutherford_jaxmarl_2024} provides baseline implementations for Overcooked~\cite{carroll_utility_2020}, another popular benchmark for cooperative multi-agent control focused on throughput efficiency.
Overcooked features two agents cooperating to complete meals.
Different layouts represent different challenges, e.g., subtask assignment and synchronization for efficiency.
\jaxmarl\ provides base implementations for independent Q-learning (IQL), VDN and QMIX (but not QPLEX), and we implemented Q+FIX-\{sum,mono,lin\}.
\cref{sec:additional-results:overcooked} contains further discussion on these tasks, and additional results.

\cref{fig:results:overcooked:returns} contains the evaluation results for the three more challenging layouts: \CoordinationRing, \ForcedCoordination, and \CounterCircuit.
In contrast to the SMACv2 results, this time it is specifically Q+FIX-mono to outperform other baselines and Q+FIX variants, indicating that there are concrete situations where Q+FIX is able to exploit a more complex fixee structure while still augmenting its performance.
Aside from this difference, these results reaffirm the ability of QFIX to greatly expand the representation capabilities of the underlying fixees, enabling higher performances.

\begin{figure}
  \centering
  \includegraphics[width=0.9\linewidth]{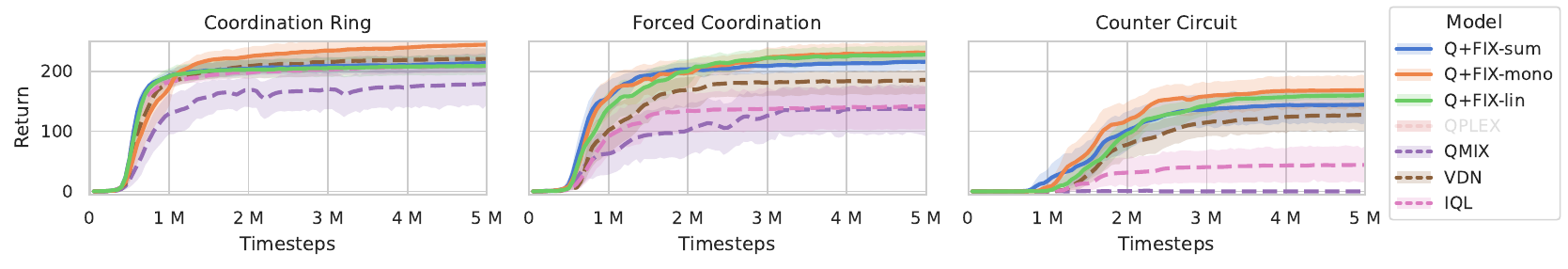}
  \caption{Overcooked return mean, bootstrapped $95\%$ CI ($20$ seeds).}
  \label{fig:results:overcooked:returns}
\end{figure}

\vspace{-0.35cm}
\section{Conclusions}
\label{sec:conclusions}
\vspace{-0.3cm}

In this work, we have advanced our understanding of the IGM function class by proposing a simple formulation of the IGM property.
From this formulation, we were able to naturally derive QFIX, a novel family of value function decomposition methods that enhance prior methods via a simple weighted transformation of their outputs, and allows the derivation and implementation of various IGM-complete models that are significantly simpler than QPLEX.
Our empirical evaluation on multiple SMACv2 and Overcooked tasks demonstrates that QFIX models succeed in
\begin{enumerate*}[label=(\roman*)]
    \item enhancing the performance of prior incomplete models like VDN and QMIX,
    \item achieving similar or better performance than QPLEX, with better convergence stability, and
    \item all this while requiring smaller mixing models.
\end{enumerate*}
Our contribution not only represents a novel approach that performs well, but also opens the door for new methods based on the QFIX framework. 

\vspace{-0.3cm}
\paragraph{Limitations}

This work expands the family of value function decomposition methods.
Though value function decomposition is very popular in the current MARL literature, it is not yet clear that such methods are necessarily the best approach to decentralized multi-agent control.
This remains an important open question, but one that falls outside of the scope of this work.

\begin{ack}
We thank Enrico Marchesini for valuable conversations and insights gained into prior work on value function decomposition.
This research was funded by NSF award 2044993.
\end{ack}

\bibliography{references}
\bibliographystyle{plainnat}


\appendix

\section{Is QPLEX truly IGM-complete?}
\label{sec:is-qplex-igm-complete}

In this section, we take a closer look at \citep[Proposition 2]{wang_dueling_2016} which appeals to the universal approximation theorem (UAT) to claim that that QPLEX is IGM-complete.
We will identify a technical issue that makes many ``strong'' forms of UAT not formally applicable, and come to the primary conclusions that
\begin{enumerate*}[label=(\roman*)]
  \item only ``weak'' forms of UAT are applicable to QPLEX, and
  \item consequently, QPLEX is able to approximate ``only'' the function class of \emph{measurable} IGM values.
\end{enumerate*}
To be clear, this is far from being a strict limitation in practice, as the class of measurable functions is extremely wide and contains any reasonable function to model, and mostly excludes deeply degenerate cases.

The main goal of this discussion is to be more specific in regards to \emph{what} version of UAT is applicable to methods like QPLEX (and QFIX), and what kinds of convergence guarantees they actually entail.

Part of the issue at hand is that UATs come in a variety of forms, each making different assumptions on the model and establishing different notions of approximation to different classes of target functions.
The UATs of \citet{cybenko_approximation_1989} and of \citet{pinkus_approximation_1999} are among the most well known, and are formulated in terms of \emph{uniform convergence}, a strong notion of approximation that is only applicable to approximate \emph{continuous} functions.
However, other forms of UAT are applicable to approximate wider classes of functions, although they are also typically associated with weaker notions of approximation.
\citet[Theorem 1]{hornik_approximation_1991} establishes a form of UAT that is applicable to functions in the Lebesgue spaces $L^p$ and entails \emph{convergence in $p$-norm}.
\citet{hornik_approximation_1991} also informally formulates a corollary that is applicable to functions that are merely \emph{measurable}, and ``only'' entails \emph{convergence in measure $\mu$}.

\paragraph{Which universal approximation theorem?}

The appeal to UAT made by \citet{wang_qplex_2020} cites a form of UAT that is analogous to those of \citet{cybenko_approximation_1989,pinkus_approximation_1999} that are formally applicable to continuous functions only.
However we note two relevant details:
\begin{enumerate*}[label=(\roman*)]
  \item QPLEX constructs $\qplex$ by composing individual values via models $w_i$, $b_i$ and $\lambda_i$; therefore any appeal to UAT must refer to these models rather than $\qplex$ as a whole.
  \item The proof of Proposition 2 is based on constructing a piece-wise target $\lambda\opt_i(\jh, \ja)$ that is clearly \emph{not} guaranteed to be continuous.
  These are
\end{enumerate*}
\begin{equation}
  \lambda\opt_i(\jh, \ja) = \begin{cases}
    \frac{1}{N} \frac{ A(\jh, \ja) }{ A_i(h_i, a_i) } & \text{when } A_i(h_i, a_i) < 0 \,, \\
    \text{any value} & \text{when } A_i(h_i, a_i) = 0 \,,
  \end{cases}
\end{equation}
where $A(\jh, \ja)$ is the advantage of the target IGM value function.
Clearly, as the target $\lambda\opt_i$ is not continuous, it is improper to appeal to a form of UAT that is based on continuous targets.

\paragraph{Resolution}

To resolve this technicality, we must find a version of UAT that is applicable to a target like $\lambda\opt_i$.
It is not immediately clear that $\lambda\opt_i$ belongs to a Lebesgue space $L^p$, or what kinds of \emph{simple} assumptions can be formulated to make it so.
As a simple resolution, we instead appeal to the weaker form of UAT by \citet{cybenko_approximation_1989} (presented informally in the discussion section) based on \emph{measurable} functions.
However, even this form of UAT still requires some technical assumptions.

To guarantee that $\lambda\opt_i$ is measurable, it is sufficient to assume that $Q_i(h_i, a_i)$ and $Q(\jh, \ja)$ are measurable functions.
Then,
\begin{itemize}
    \item $V_i(h_i)$, $V(\jh)$, $A_i(h_i, a_i)$, $A(\jh, \ja)$ are measurable;
    \item $\argmax_{h_i, a_i} A_i(h_i, a_i)$ (the preimage of $A_i(h_i, a_i) = 0$) is a measurable set;
    \item $\lambda\opt_i$ is a piece-wise function defined by combining measurable functions partitioned in (two) measurable sets, and is therefore also measurable.
\end{itemize}

This is sufficient to guarantee convergence to $\lambda\opt_i$ in measure.
Technically this assumption means that there are \emph{non-measurable} IGM values that cannot be approximated by QPLEX (nor QFIX).
However, we reiterate that this is not a practical concern as
\begin{enumerate*}[label=(\roman*)]
  \item they represent an insignificant subset of all IGM values, and 
  \item they are degenerate and unlikely to match realistic and desirable notions of values.
\end{enumerate*}

\section{Proofs}
\label{sec:proofs}

\subsection{Proof of \cref{thm:qigm:igm}}
\label{sec:proof:qigm:igm}

We prove the two statements separately.

\paragraph{$\qigm$ satisfies IGM}

\begin{proof}

  For any given joint history $\jh$, let $a\opt_i \in \argmax_{a_i} Q_i(h_i, a_i)$ denote any maximal action according to the individual utilities, and let $\ja\opt = (a\opt_1, \ldots, a\opt_N)$ be a joint action constructed accordingly.
  For any $\ja\opt$ constructed this way, the corresponding advantage utilities are zero $\forall i \, (u\opt_i = 0)$, and
  \begin{align}
    \qigm(\jh, \ja\opt) &= w(\jh, \ja\opt) \underbrace{ f(u\opt_1, \ldots, u\opt_N) }_{=0} + b(\jh) \nonumber \\
    &= b(\jh) \,.
  \end{align}

  For any other $\ja$, we have at least one strictly negative utility $\exists i \, (u_i < 0)$, and
  \begin{align}
    \qigm(\jh, \ja) &= \underbrace{ w(\jh, \ja) }_{>0} \underbrace{ f(u_1, \ldots, u_N) }_{<0} + b(\jh) \nonumber \\
    &< b(\jh) \,.
  \end{align}

  Therefore $\ja\opt \in \argmax_{\ja} \qigm(\jh, \ja)$, and the actions that maximize the individual utilities also maximize the joint value.

\end{proof}

\paragraph{$\qigm$ is IGM-complete}

\begin{proof}[Proof by mutual inclusion]

  Let us denote the function class of $\qigm$ as $\fc(\qigm)$, and the IGM-complete function class as $\fcigm$.
  We prove $\fc(\qigm) = \fcigm$ by mutual inclusion:
  \begin{enumerate}
    \item $Q \in\fc(\qigm) \implies Q \in\fcigm$, i.e., $\qigm$ satisfies IGM (already proven above),
    \item $Q \in\fcigm \implies Q \in\fc(\qigm)$, i.e., any IGM function is representable by $\qigm$.
  \end{enumerate}

  Step 1 was already proven earlier.
  Next, we prove step 2.



  Let $Q_i(h_i, a_i)$ and $Q(\jh, \ja)$ denote an arbitrary set of individual and joint values that satisfy IGM, i.e., $Q\in \fcigm$.
  Let us denote the usual corresponding individual values and advantages as follows,
  \begin{align}
    V_i(h_i) &= \max_{a_i} Q_i(h_i, a_i) \,, &
    A_i(h_i, a_i) &= Q_i(h_i, a_i) - V_i(h_i) \,, \\
    V(\jh) &= \max_{\ja} Q(\jh, \ja) \,, &
    A(\jh, \ja) &= Q(\jh, \ja) - V(\jh) \,,
  \end{align}
  with the usual shorthand $q_i = Q_i(h_i, a_i)$ and $v_i = V_i(h_i)$, and $u_i = A_i(h_i, a_i)$.

  For any $f$ that satisfies the requirements of \cref{eq:qigm}, let $w$ and $b$ be defined as follows,
  \begin{align}
    b(\jh) &= V(\jh) \,, \\
    w(\jh, \ja) &= \begin{cases}
      \frac{ A(\jh, \ja) }{ f(u_1, \ldots, u_N) } \,, & \text{if } f(u_1, \ldots, u_N) \neq 0 \,, \\
      \text{any value} \,, & \text{otherwise} \,.
    \end{cases}
  \end{align}

  For any given joint history $\jh$, let $a\opt_i \in \argmax_{a_i} Q_i(h_i, a_i)$ denote a maximal action according to the individual utilities, and $\ja\opt = (a\opt_1, \ldots, a\opt_N)$ the corresponding joint action.
  Given that $Q$ satisfies IGM by assumption, we have $\ja\opt \in \argmax_{\ja} Q(\jh, \ja)$, and $Q(\jh, \ja\opt) = \max_{\ja} Q(\jh, \ja) = V(\jh)$.

  For any $\ja\opt$ constructed this way, the corresponding advantage utilities are zero $\forall i \, (u_i = 0)$, and
  \begin{align}
    \qigm(\jh, \ja\opt) &= w(\jh, \ja\opt) f(u_1, \ldots, u_N) + b(\jh) \nonumber \\
                        &= w(\jh, \ja\opt) \underbrace{ f(0, \ldots, 0) }_{=0} + b(\jh) \nonumber \\
                        &= V(\jh) \nonumber \\
                        &= Q(\jh, \ja\opt) \,.
  \end{align}

  For any other $\ja$, we have at least one strictly negative utility $\exists i \, (u_i < 0)$, and
  \begin{align}
    \qigm(\jh, \ja) &= w(\jh, \ja) f(u_1, \ldots, u_N) + b(\jh) \nonumber \\
                    &= \frac{ A(\jh, \ja) }{ f(u_1, \ldots, u_N) } f(u_1, \ldots, u_N) + V(\jh) \nonumber \\
                    &= A(\jh, \ja) + V(\jh) \nonumber \\
                    &= Q(\jh, \ja) \,.
  \end{align}

  In either case, $\qigm(\jh, \ja) = Q(\jh, \ja)$ for all inputs.
  Therefore $Q \in\fcigm \implies Q \in\fc(\qigm)$.

\end{proof}

\subsection{Proof of \cref{thm:qfix}}
\label{sec:proof:qfix}

\begin{proof}
  \Cref{eq:qfix} satisfies the form and requirements of \cref{eq:qigm}.
  Therefore, IGM follows from \cref{thm:qigm:igm}.
  Assuming target IGM values that are measurable, then the targets constructed in the proof of \cref{thm:qigm:igm} are also measurable, and we can appeal to the universal approximation theorems of \citet{hornik_approximation_1991} to show that $w$, $b$ are able to approximate such targets.
  (also see \cref{sec:is-qplex-igm-complete} for a similar discussion relating to QPLEX).
\end{proof}

\subsection{Proof of \cref{thm:qfix-lin}}
\label{sec:proof:qfix-lin}

\begin{proof}
QFIX-lin is a monotonic function of individual advantages and therefore satisfies IGM.
QFIX-lin is also a generalization of QFIX-sum, therefore its function class is a superset of the QFIX-sum function class, i.e., the class of measurable IGM values.
Therefore, QFIX-lin can represent all measurable functions that satisfy IGM, and none of those that do not satisfy IGM.
\end{proof}

%

\section{Derivations}
\label{sec:derivations}

This section contains explicit long-form derivations that had to be removed from the main document due to space limitations.

\subsection{VDN maximal values and advantages}
\label{sec:derivations:vdn}

As a reminder, VDN action-values are defined as $\qvdn(\jh, \ja) \doteq \sum_i \qmodel_i(h_i, a_i)$.
Due to the the linear (monotonic) mixing structure, the joint maximal values $\vvdn(\jh)$ can be expressed as the sum of the individual maximal values,
\begin{align}
  \vvdn(\jh) &\doteq \max_{\ja} \qvdn(\jh, \ja) \nonumber \\
  &= \max_{a_1, \ldots, a_N} \sum_i \qmodel_i(h_i, a_i) \nonumber \\
  &= \sum_i \max_{a_i} \qmodel_i(h_i, a_i) & \text{(monotonicity)} \nonumber \\
  \label{eq:vvdn}
  &= \sum_i \vmodel_i(h_i) \,,
  \intertext{and the joint advantages $\avdn(\jh, \ja)$ can be expressed as the sum of the individual advantages,}
  \avdn(\jh, \ja) &\doteq \qvdn(\jh, \ja) - \vvdn(\jh) \nonumber \\
  &= \sum_i \qmodel_i(h_i, a_i) - \sum_i \vmodel_i(h_i) \nonumber \\
  &= \sum_i \qmodel_i(h_i, a_i) - \vmodel_i(h_i) \nonumber \\
  \label{eq:avdn}
  &= \sum_i \amodel_i(h_i, a_i) \,.
\end{align}

\subsection{QMIX maximal values and advantages}
\label{sec:derivations:qmix}

As a reminder, QMIX action-values are defined as $\qmix(\jh, \ja) \doteq \fmono\left( q_1, \ldots, q_N \right)$.
Due to the monotonic mixing structure, the joint maximal values $\vmix(\jh)$ can be expressed as the monotonic mixing of the individual maximal values,
\begin{align}
  \vmix(\jh) &\doteq \max_{\ja} \qmix(\jh, \ja) \nonumber \\
  &= \max_{a_1, \ldots, a_N} \fmono\left( \qmodel_1(h_1, a_1), \ldots, \qmodel_N(h_N, a_N) \right) \nonumber \\
  &= \fmono\left( \max_{a_1} \qmodel_1(h_1, a_1), \ldots, \max_{a_N} \qmodel_N(h_N, a_N) \right) & \text{(monotonicity)} \nonumber \\
  &= \fmono\left( \vmodel_1(h_1), \ldots, \vmodel_N(h_N) \right) \nonumber \\
  \label{eq:vmix}
  &= \fmono\left( v_1, \ldots, v_N \right) \,,
  \intertext{and the joint advantages $\amix(\jh, \ja)$ can be expressed as the corresponding difference,}
  \label{eq:amix}
  \amix(\jh, \ja) &\doteq \qmix(\jh, \ja) - \vmix(\jh) \nonumber \\
                  &= \fmono\left( q_1, \ldots, q_N \right) - \fmono\left( v_1, \ldots, v_N \right) \,.
\end{align}

\subsection{QFIX-sum}
\label{sec:derivations:qfix-sum}

QFIX-sum is an instance of QFIX based on VDN as fixee model, $\qfixee(\jh, \ja) = \qvdn(\jh, \ja)$.
From \cref{eq:avdn}, we have that the VDN joint advantage is given as the sum of individual advantages (hence the ``-sum'' suffix).
Therefore, QFIX-sum is simply obtained as
\begin{align}
  \qfixsum(\jh, \ja) &\doteq w(\jh, \ja) \avdn(\jh, \ja) + b(\jh) \nonumber \\
  &= w(\jh, \ja) \sum_i \amodel_i(h_i, a_i) + b(\jh) \,.
\end{align}

\subsection{QFIX-mono}
\label{sec:derivations:qfix-mono}

QFIX-mono is an instance of QFIX based on QMIX as fixee model, $\qfixee(\jh, \ja) = \qmix(\jh, \ja)$.
From \cref{eq:amix}, we have that the QMIX advantage is given as a difference between monotonic compositions of individual utilities (hence the ``-mono'' suffix).
Therefore, QFIX-mono is simply obtained as
\begin{align}
  \qfixmono(\jh, \ja) &\doteq w(\jh, \ja) \amix(\jh, \ja) + b(\jh) \nonumber \\
                      &= w(\jh, \ja) ( \fmono(q_1, \ldots, q_N) - \fmono(v_1, \ldots, v_N) ) + b(\jh) \,.
\end{align}

\subsection{Q+FIX-sum}
\label{sec:derivations:q+fix-sum}

\begin{figure*}
  \centering

  \begin{subfigure}{0.49\linewidth}
    \centering
    \includegraphics[width=\linewidth]{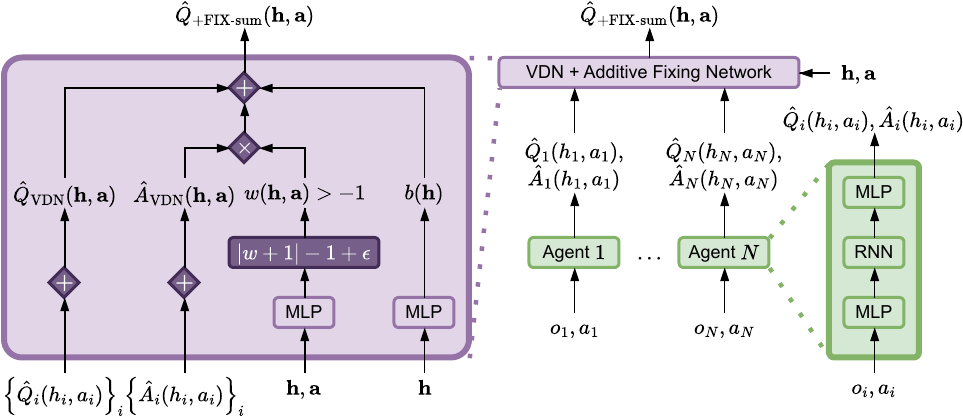}
    \caption{Q+FIX-sum diagram.}
    \label{fig:q+fix-sum}
  \end{subfigure}
  \begin{subfigure}{0.49\linewidth}
    \centering
    \includegraphics[width=\linewidth]{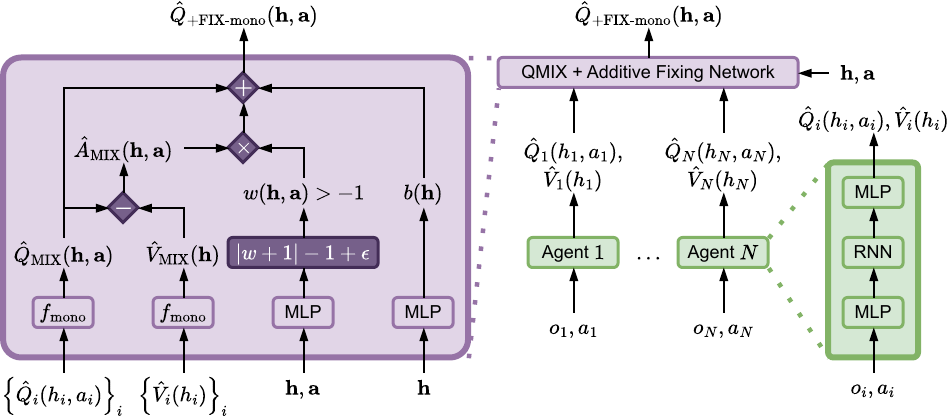}
    \caption{Q+FIX-mono diagram.}
    \label{fig:q+fix-mono}
  \end{subfigure}
  
  \begin{subfigure}{\linewidth}
    \centering
    \includegraphics[width=0.49\linewidth]{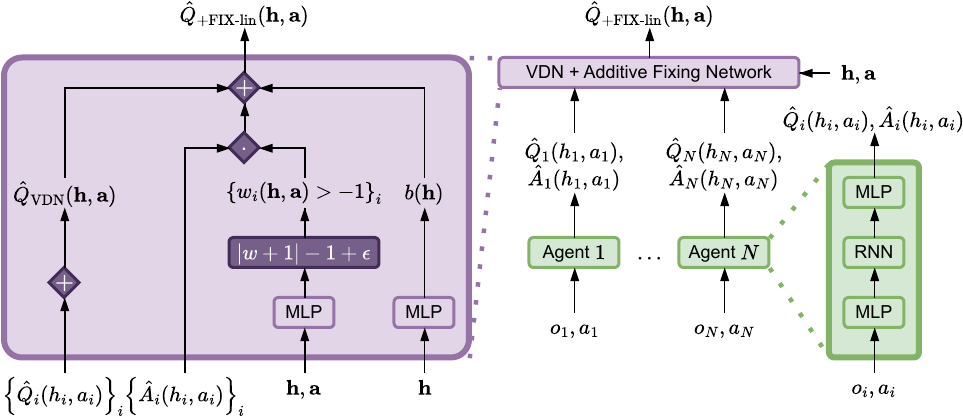}
    \caption{Q+FIX-lin diagram.}
    \label{fig:q+fix-lin}
  \end{subfigure}

  \caption{Specialized diagrams for Q+FIX-sum, Q+FIX-mono, and Q+FIX-lin.}
  \label{fig:q+fix:specializations}
\end{figure*}

Q+FIX-sum is an instance of Q+FIX based on VDN as fixee model, $\qfixee(\jh, \ja) = \qvdn(\jh, \ja)$ and $\afixee(\jh, \ja) = \avdn(\jh, \ja)$, also equivalent to the additive formulation of QFIX-sum.
Therefore, Q+FIX-sum is simply obtained as
\begin{align}
\aqfixsum{} &\doteq \qvdn(\jh, \ja) + w(\jh, \ja) \avdn(\jh, \ja) + b(\jh) \nonumber \\
&\doteq \sum_i \qmodel_i(\jh, \ja) + w(\jh, \ja) \sum_i \amodel_i(\jh, \ja) + b(\jh) \,.
\end{align}
\Cref{fig:q+fix-sum} shows a graphical diagram for Q+FIX-sum.

\subsection{Q+FIX-mono}
\label{sec:derivations:q+fix-mono}

Q+FIX-mono is an instance of Q+FIX based on QMIX as fixee model, $\qfixee(\jh, \ja) = \qmix(\jh, \ja)$ and $\afixee(\jh, \ja) = \amix(\jh, \ja)$, also equivalent to the additive formulation of QFIX-mono.
Therefore, Q+FIX-mono is simply obtained as
\begin{align}
\aqfixmono{} &\doteq \qvdn(\jh, \ja) + w(\jh, \ja) \avdn(\jh, \ja) + b(\jh) \nonumber \\
&\doteq \fmono(q_1, \ldots, q_N) + w(\jh, \ja) \left( \fmono(q_1, \ldots, q_N) - \fmono(v_1, \ldots, v_N) \right) + b(\jh) \,.
\end{align}
\Cref{fig:q+fix-mono} shows a graphical diagram for Q+FIX-mono.

\subsection{Q+FIX-lin}
\label{sec:derivations:q+fix-lin}

Q+FIX-lin is the additive formulation of QFIX-lin.
Just as QFIX-lin is not formally a member of the QFIX family, but rather a generalization of QFIX-sum, so is Q+FIX-lin not formally a member of Q+FIX, but rather a generalization of Q+FIX-sum.
Given that QFIX-lin is obtained by introducing per-agent weights $w_i(\jh, \ja)$, Q+FIX-lin is simply obtained as
\begin{align}
\aqfixlin{} &\doteq \sum_i \qmodel_i(h_i, a_i) + \sum_i w_i(\jh, \ja) \amodel_i(h_i, a_i) + b(\jh) \nonumber \,.
\end{align}
\Cref{fig:q+fix-lin} shows a graphical diagram for Q+FIX-lin.

\section{Why does detaching the advantages help Q+FIX?}
\label{sec:detach-hypothesis}

First, we note that the gradients $\nabla_{\theta_i} \aqfix(\jh, \ja)$ when the advantages \emph{are not} detached are
\begin{align}
  \nabla_{\theta_i} \aqfix(\jh, \ja) &= \nabla_{\theta_i} \qfixee(\jh, \ja) + w(\jh, \ja) \nabla_{\theta_i} \afixee(\jh, \ja) \nonumber \\
  &= \nabla_{\theta_i} \vfixee(\jh) + ( w(\jh, \ja) + 1 ) \nabla_{\theta_i} \afixee(\jh, \ja) \,.
\end{align}

It seems plausible that there may be values of $w(\jh, \ja)$ that could result in non-ideal gradient signals.
For example, a low fixing weight $w(\jh, \ja) \approx -1$ results in a dampened gradient $\nabla_{\theta_i} \aqfix(\jh, \ja) \approx \nabla_{\theta_i} \vfixee(\jh)$, that is notably independent on actions.
On the other end of the spectrum, a very large fixing weight $w(\jh, \ja) \gg -1$ results in a gradient that is dominated by the highly-weighted advantage component, overcoming the value component, $\nabla_{\theta_i} \aqfix(\jh, \ja) \approx w(\jh, \ja) \nabla_{\theta_i} \afixee(\jh, \ja)$.
On each end of the spectrum, the gradient will propagate almost exclusively through the values $\nabla_{\theta_i} \vfixee(\jh)$ or through the advantages $\nabla_{\theta_i} \afixee(\jh, \ja)$.

On the other hand, the gradients $\nabla_{\theta_i} \aqfix(\jh, \ja)$ when the advantages \emph{are} detached are
\begin{align}
  \nabla_{\theta_i} \aqfix(\jh, \ja) &= \nabla_{\theta_i} \qfixee(\jh, \ja) \nonumber \\
  &= \nabla_{\theta_i} \vfixee(\jh) + \nabla_{\theta_i} \afixee(\jh, \ja) \,,
\end{align}
and are invariant to the fixing structure, equally dependent on the value and advantage components.

\section{Stateful QFIX}
\label{sec:appendix:stateful-qfix}

In this section, we extend some of the theory of QFIX to the stateful case.
As mentioned in the main document, we consider two cases of stateful QFIX, a \emph{history-state} case and \emph{state-only} case, which differ in what information is provided to the fixing network.
The derivations and proofs will follow closely those of the stateless case, although not all conclusions will transfer to all stateful cases.
Primarily, we will find that state-only QFIX (like other state-only variants of other methods) is not able to represent the full IGM-complete space of value functions.

\subsection{History-state QFIX}
\label{sec:appendix:history-state-qfix}

Consider a history-state variant of $\qigm$ from \cref{eq:qigm} defined as follows,
\begin{align}
    \qigm(\jh, s, \ja) \doteq w(\jh, s, \ja) f(u_1, \ldots, u_N) + b(\jh, s) \,,
\end{align}
where $u_i$ and $f$ are defined as in \cref{sec:qigm}, $w \colon\jhset\times\sset\times\jaset \to\realset_{>0}$ is an arbitrary positive function of joint history, state, and joint action, $b \colon\jhset\times\sset \to\realset$ is an arbitrary function of joint history and state.
As in the stateless case, $\qigm(\jh, s, \ja)$ denotes a relationship where any deviation from individual maximality is transformed into an arbitrary deviation from joint maximality.

\begin{proposition}
    \label{thm:history-state-qigm:igm}
    For any $f$, $w$, and $b$, values $\{ Q_i \}\fori$ and $\qigm$ satisfy stateful-IGM.
\end{proposition}

\begin{proof}
    This proof follows the same structure as that for \cref{thm:qigm:igm}.
    
    For any given joint history $\jh$, let $a\opt_i = \argmax_{a_i} Q_i(h_i, a_i)$ denote the maximal action according to the individual utilities, and $\ja\opt = (a\opt_i, \ldots, a\opt_N)$ the joint action constructed by those individual actions.
    We prove that $\qigm$ satisfies stateful-IGM in two steps:
    
    \begin{enumerate}
        \item $\ja\opt = \argmax_\ja \qigm(\jh, s, \ja)$, i.e., 
        the individual maximal actions also maximize the joint history-state values.
        \item $\ja\opt = \argmax_\ja \Exp_{s\mid \jh}\left[ \qigm(\jh, s, \ja) \right]$, i.e., the individual maximal actions also maximize the marginalized joint history-state values.
    \end{enumerate}

    \paragraph{Step 1.}
    
    The advantage utilities corresponding to $\ja\opt$ are zero $\forall i (u_i = 0)$ by definition, and
    \begin{align}
        \label{eq:history-state-qigm}
        \qigm(\jh, s, \ja\opt) &= w(\jh, s, \ja\opt) \underbrace{ f(u_1, \ldots, u_N) }_{=0} + b(\jh, s) \nonumber \\
        &= b(\jh, s) \,.
    \end{align}

    For any other non-maximal action $\ja$, we have at least one strictly negative utility $\exists i (u_i < 0)$, and
    \begin{align}
        \qigm(\jh, s, \ja\opt) &= \underbrace{ w(\jh, s, \ja\opt) }_{>0} \underbrace{ f(u_1, \ldots, u_N) }_{<0} + b(\jh, s) \nonumber \\
        &< b(\jh, s) \,.
    \end{align}

    Therefore, $\ja\opt = \argmax_\ja \qigm(\jh, s, \ja)$, and the actions that maximize the individual utilities also maximize the joint history-state value.
    
    \paragraph{Step 2.}

    Note that $\ja\opt = \argmax_\ja \qigm(\jh, s, \ja)$ is valid for any state, at the very least because $\ja\opt$ are defined via the stateless individual utilities.

    If $\ja\opt$ maximizes the joint history-state values for any given state, then it also maximizes the joint history-state values when marginalized over any distribution of state $p\in\Delta\sset$, and $\ja\opt = \argmax_\ja \Exp_{s\sim p}\left[ \qigm(\jh, s, \ja) \right]$.
    This must be true also for the specific distribution $p(s) \doteq \Pr(s\mid \jh)$, and $\ja\opt = \argmax_\ja \Exp_{s\mid \jh}\left[ \qigm(\jh, s, \ja) \right]$.

    Therefore, the same actions $\ja\opt$ that maximize the individual utilities, also maximize the marginalized joint history-state values, satisfying the definition of stateful-IGM in \cref{thm:stateful-igm}.
    
\end{proof}

When it comes to a stateful form of IGM-complete function class, we must be very clear as to what it is that we are able to prove.
We are not able to prove that $\qigm(\jh, s, \ja)$ covers the whole stateful-IGM function class of values that satisfy stateful-IGM (we do not believe this is possible, though we will not go into that amount of detail here).
Instead, we prove that the projected space of \emph{stateless} values obtained by marginalizing the stateful values via $\Exp_{s\mid \jh}\left[ \qigm(\jh, s, \ja) \right]$ is the IGM-complete function class.

\begin{proposition}
    \label{thm:history-state-qigm:igm-complete}
    
    For any $f$, and given free choice of $w$ and $b$, the function class of $\{ Q_i \}\fori$ and projected $\Exp_{s\mid h}\left[ \qigm \right]$ is IGM-complete.
\end{proposition}

\begin{proof}
    This proof follows the same structure as that for \cref{thm:qigm:igm}, although we consider the projected space stateless values $\Exp_{s\mid \jh}\left[ \qigm(\jh, s, \ja) \right]$ obtained from the stateful values $\qigm(\jh, s, \ja)$.

    Let us denote the projected function class of $\qigm$ as $\fc(\qigm)$, and the stateful IGM-complete function class as $\fcigm$.
    We prove the equivalence $\fc(\qigm) = \fcigm$ in two steps:

    \paragraph{Step 1.}
  
    $Q \in\fc(\qigm) \implies Q \in\fcigm$ follows directly from \cref{thm:history-state-qigm:igm}.
  
    \paragraph{Step 2.}
  
    Let $Q_i(h_i, a_i)$ and $Q(\jh, \ja)$ denote an arbitrary set of individual and joint values that satisfy IGM, i.e., $Q\in \fcigm$.
    Let us denote the usual corresponding values and advantages as follows,
    \begin{align}
      V_i(h_i) &= \max_{a_i} Q_i(h_i, a_i) \,, &
      A_i(h_i, a_i) &= Q_i(h_i, a_i) - V_i(h_i) \,, \\
      \intertext{\emph{but}, let us define a different notion of joint values and advantages for this history-state case (note the stateless $V$, stateful $A$),}
      V(\jh) &= \max_{\ja} Q(\jh, \ja) \,, &
      A(\jh, \ja) &= Q(\jh, \ja) - V(\jh) \,,
    \end{align}
    with the usual shorthand $q_i = Q_i(h_i, a_i)$ and $v_i = V_i(h_i)$, and $u_i = A_i(h_i, a_i)$.
  
    For any $f$ that satisfies the requirements of \cref{eq:history-state-qigm}, let $w$ and $b$ be defined as follows,
    \begin{align}
      b(\jh, s) &= V(\jh) \,, \\
      w(\jh, s, \ja) &= \begin{cases}
        \frac{ A(\jh, \ja) }{ f(u_1, \ldots, u_N) } \,, & \text{if } f(u_1, \ldots, u_N) \neq 0 \,, \\
        \text{any value} \,, & \text{otherwise} \,.
      \end{cases}
    \end{align}

    These definitions effectively create stateful values $\qigm(\jh, s, \ja)$ that are state-independent, and functionally equivalent to stateless values $\qigm(\jh, \ja)$.
    Although this appears to be a severe misuse of the additional state information, it is sufficient to prove the claim that the projected space of stateless values obtained via marginalization $\Exp_{s\mid \jh}\left[ \qigm(\jh, \ja) \right]$ is IGM-complete.
    It's easy to see that the rest of the proof can not proceed as in \cref{thm:qigm:igm}.
  
    For any given joint history $\jh$, let $a\opt_i = \argmax_{a_i} Q_i(h_i, a_i)$ denote the maximal action according to the individual utilities, and $\ja\opt = (a\opt_1, \ldots, a\opt_N)$ the corresponding joint action.
    Given that $Q$ satisfies IGM by assumption, we have $\ja\opt = \argmax_{\ja} Q(\jh, \ja)$, and $Q(\jh, \ja\opt) = \max_{\ja} Q(\jh, \ja) = V(\jh)$.
  
    For this joint action $\ja\opt$, the corresponding individual advantage utilities are zero $\forall i \, (u_i = 0)$ by definition, and
    \begin{align}
      \qigm(\jh, s, \ja\opt) &= w(\jh, s, \ja\opt) f(u_1, \ldots, u_N) + b(\jh, s) \nonumber \\
                          &= w(\jh, s, \ja\opt) \underbrace{ f(0, \ldots, 0) }_{=0} + b(\jh, s) \nonumber \\
                          &= V(\jh) \nonumber \\
                          &= Q(\jh, \ja\opt) \,.
    \end{align}
  
    For any other non-maximal action $\ja\D$, we have at least one strictly negative utility $\exists i \, (u_i < 0)$, and
    \begin{align}
      \qigm(\jh, s, \ja\D) &= w(\jh, s, \ja\D) f(u_1, \ldots, u_N) + b(\jh, s) \nonumber \\
                        &= \frac{ A(\jh, \ja\D) }{ f(u_1, \ldots, u_N) } f(u_1, \ldots, u_N) + V(\jh) \nonumber \\
                        &= A(\jh, \ja\D) + V(\jh) \nonumber \\
                        &= Q(\jh, \ja\D) \,.
    \end{align}
  
    In either case, $\qigm(\jh, s, \ja) = Q(\jh, \ja)$ for all joint histories, states, and actions, which trivially implies $\Exp_{s\mid \jh}\left[ \qigm(\jh, s, \ja) \right] = Q(\jh, \ja)$.
    Therefore $Q \in\fcigm \implies Q \in\fc(\qigm)$.
    
\end{proof}

\subsection{State-only QFIX}
\label{sec:appendix:state-only-qfix}

Consider a state-only variant of $\qigm$ from \cref{eq:qigm} defined as follows,
\begin{align}
    \qigm(\jh, s, \ja) \doteq w(s, \ja) f(u_1, \ldots, u_N) + b(s) \,,
\end{align}
where $u_i$ and $f$ are defined as in \cref{sec:qigm}, $w \colon\sset\times\jaset \to\realset_{>0}$ is an arbitrary positive function of joint history, state, and joint action, $b \colon\sset \to\realset$ is an arbitrary function of joint history and state.
As in the stateless case, $\qigm(\jh, s, \ja)$ denotes a relationship where any deviation from individual maximality is transformed into an arbitrary deviation from joint maximality.
Note that the name \emph{state-only} refers moreso to the fixing models $w, b$, than the values as a whole that remain at least in part history-based due to the dependence on the individual history-based utilities.

\begin{proposition}
    \label{thm:state-only-qigm:igm}
    For any $f$, $w$, and $b$, values $\{ Q_i \}\fori$ and $\qigm$ satisfy stateful-IGM.
\end{proposition}

\begin{proof}
    This proof follows the same structure as that for \cref{thm:qigm:igm}.
    
    For any given joint history $\jh$, let $a\opt_i = \argmax_{a_i} Q_i(h_i, a_i)$ denote the maximal action according to the individual utilities, and $\ja\opt = (a\opt_i, \ldots, a\opt_N)$ the joint action constructed by those individual actions.
    We prove that $\qigm$ satisfies stateful-IGM in two steps:
    
    \begin{enumerate}
        \item $\ja\opt = \argmax_\ja \qigm(\jh, s, \ja)$, i.e., 
        the individual maximal actions also maximize the state-only values.
        \item $\ja\opt = \argmax_\ja \Exp_{s\mid \jh}\left[ \qigm(\jh, s, \ja) \right]$, i.e., the individual maximal actions also maximize the marginalized joint state-only values.
    \end{enumerate}

    \paragraph{Step 1.}
    
    The advantage utilities corresponding to $\ja\opt$ are zero $\forall i (u_i = 0)$ by definition, and
    \begin{align}
        \qigm(\jh, s, \ja\opt) &= w(s, \ja\opt) \underbrace{ f(u_1, \ldots, u_N) }_{=0} + b(s) \nonumber \\
        &= b(s) \,.
    \end{align}

    For any other non-maximal action $\ja$, we have at least one strictly negative utility $\exists i (u_i < 0)$, and
    \begin{align}
        \qigm(\jh, s, \ja\opt) &= \underbrace{ w(s, \ja\opt) }_{>0} \underbrace{ f(u_1, \ldots, u_N) }_{<0} + b(s) \nonumber \\
        &< b(s) \,.
    \end{align}

    Therefore, $\ja\opt = \argmax_\ja \qigm(\jh, s, \ja)$, and the actions that maximize the individual utilities also maximize the joint state-only value.
    
    \paragraph{Step 2.}

    Note that $\ja\opt = \argmax_\ja \qigm(\jh, s, \ja)$ is valid for any state, at the very least because $\ja\opt$ are defined via the stateless individual utilities.

    If $\ja\opt$ maximizes the joint history-state values for any given state, then it also maximizes the joint history-state values when marginalized over any distribution of state $p\in\Delta\sset$, and $\ja\opt = \argmax_\ja \Exp_{s\sim p}\left[ \qigm(\jh, s, \ja) \right]$.
    This must be true also for the specific distribution $p(s) \doteq \Pr(s\mid \jh)$, and $\ja\opt = \argmax_\ja \Exp_{s\mid \jh}\left[ \qigm(\jh, s, \ja) \right]$.

    Therefore, the same actions $\ja\opt$ that maximize the individual utilities, also maximize the marginalized joint history-state values, satisfying the definition of stateful-IGM in \cref{thm:stateful-igm}.
    
\end{proof}

In contrast to history-state QFIX in \cref{sec:appendix:history-state-qfix}, we are not able to prove that state-only QFIX is able to represent the complete function class of IGM values.

\section{Evaluation details and additional results}

\subsection{SMACv2}
\label{sec:additional-results:smacv2}

\paragraph{Implementation details}

We note that \pymarl\
provides \emph{stateful} implementations of QMIX and QPLEX.
For QPLEX in particular, this means that state-only weights $w_i(s)$ and $\lambda_i(s, \ja)$ are employed.
As discussed by \citet{marchesini_stateful_2024}, the state-only implementation of QPLEX loses some of the theoretical properties related to full IGM-completeness (and the same holds for Q+FIX, see \cref{sec:appendix:stateful-qfix}).
However, to maintain a fair comparison, our implementation of Q+FIX employs analogous stateful implementation with state-only weights $w(s, \ja)$ for Q+FIX-\{sum,mono\}, and $w_i(s, \ja)$ for Q+FIX-lin.
QPLEX and Q+FIX implementations both employ \emph{advantage detaching} as previously described.
For these SMACV2 experiments, we did not find it necessary to employ \emph{intervention annealing}.

\paragraph{Metrics}

SMACv2 logs various metrics pertaining to team performance, including the mean return and the mean winrate obtained as the ratio of episodes where the agents succeed in defeating the enemies.
Although the winrate is a common metric used in prior work (e.g., \citet{wang_qplex_2020} use the winrate in their SMACv1 evaluation), we have found that winrates induce a different ordering over performances, i.e., it is possible to obtain a higher winrate while achieving a lower return, and vice versa.
This indicates that the rewards of SMACv2 do not perfectly encode the task of defeating the enemies---a matter of reward design that is beyond the scope of this work.
Since returns are the metric that the methods are directly trained to maximize, we prioritize returns as our primary evaluation metric in the main document, but also provide winrate results in this appendix.

\paragraph{Winrate results}

\begin{figure}
  \begin{subfigure}{\linewidth}
    \centering
    \includegraphics[width=\linewidth]{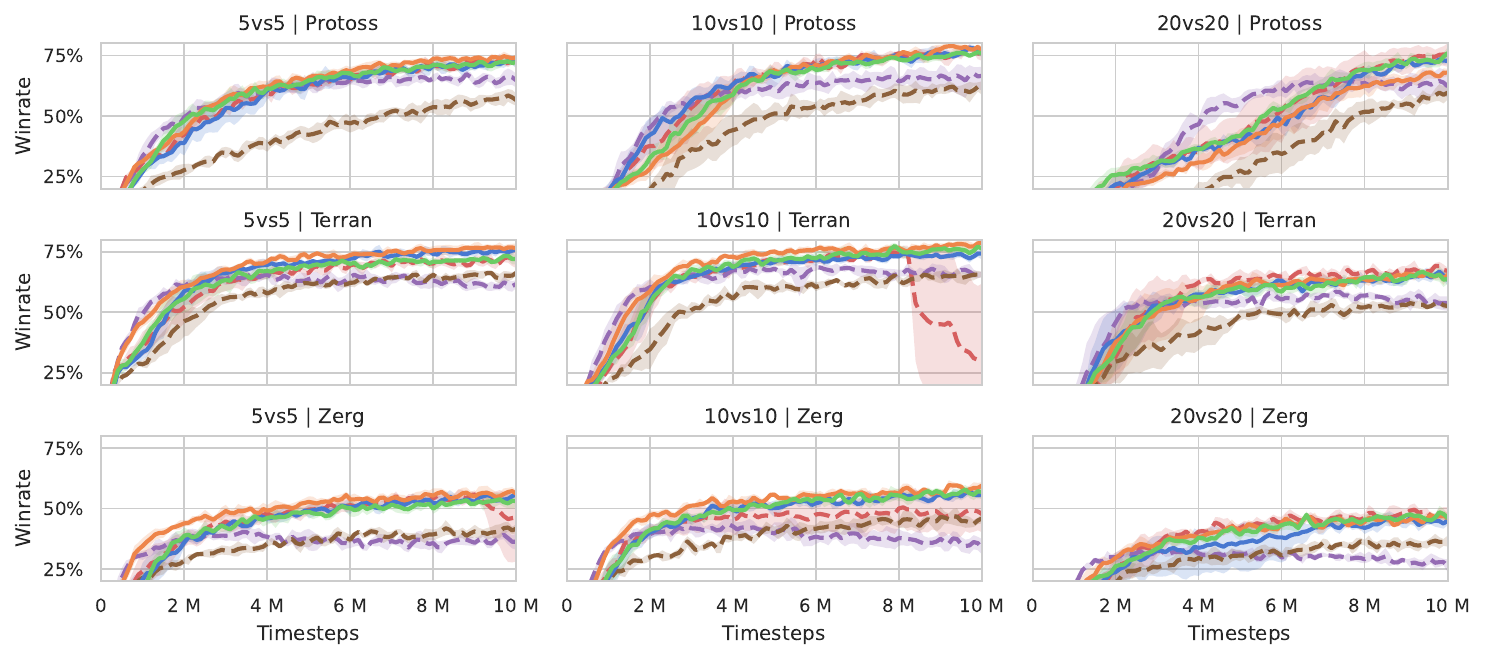}
    \caption{SMACv2 winrate mean ($5$ seeds).}
    \label{fig:results:smacv2:mean-winrates}
  \end{subfigure}

  \begin{subfigure}{0.44\linewidth}
    \centering
    \includegraphics[height=2.4cm]{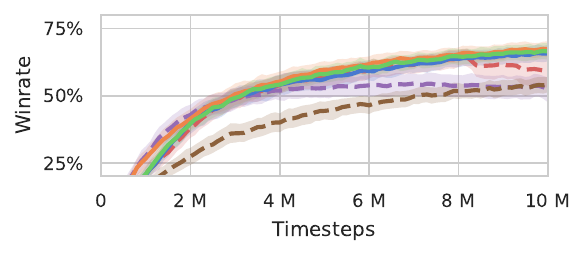}
    \caption{SMACv2 winrate mean aggregate ($45$ seeds).}
    \label{fig:results:smacv2:mean-winrates:aggregate}
  \end{subfigure}
  \begin{subfigure}{0.54\linewidth}
    \centering
    \includegraphics[height=2.4cm]{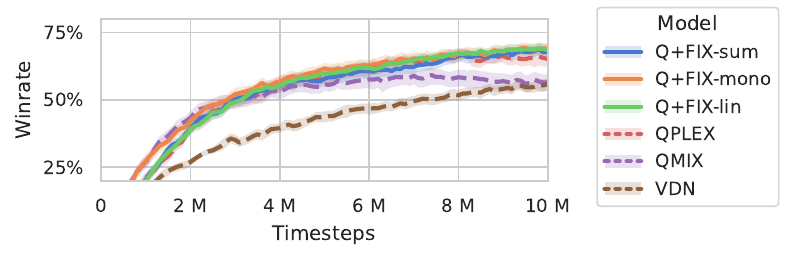}
    \caption{SMACv2 winrate IQM aggregate ($45$ seeds).}
    \label{fig:results:smacv2:iqm-winrates:aggregate}
  \end{subfigure}

  \caption{SMACv2 winrate results, bootstrapped $95\%$ CI.}
  \label{fig:results:smacv2:winrates}
\end{figure}

In this section, we show additional results based on the winrate metric.
As with the return-based results, we show the learning performance for each model and scenario in \cref{fig:results:smacv2:mean-winrates}, and the aggregate winrate across scenarios in \cref{fig:results:smacv2:mean-winrates:aggregate}.

\paragraph{Winrates vs returns}

As mentioned in the main document, the winrate and return metrics induce correlated but notably different orderings over the evaluated methods.
Comparing \cref{fig:results:smacv2:returns,fig:results:smacv2:winrates}, this is notable by the following (non-exhaustive) observations:
\begin{itemize}
    \item In \Terranv,
    \begin{itemize}
        \item Return indicates Q+FIX-sum $\succ$ Q+FIX-mono.
        \item Winrate indicates Q+FIX-sum $\prec$ Q+FIX-mono.
    \end{itemize}
    \item In \Zergv,
    \begin{itemize}
        \item Return indicates Q+FIX-sum $\succ$ Q+FIX-mono $\approx$ Q+FIX-lin.
        \item Winrate indicates Q+FIX-sum $\approx$ Q+FIX-mono $\approx$ Q+FIX-lin.
    \end{itemize}
    \item In \Zergx,
    \begin{itemize}
        \item Return indicates VDN $\approx$ Q+FIX.
        \item Winrate indicates VDN $\prec$ Q+FIX.
    \end{itemize}
    \item In \Protossxx,
    \begin{itemize}
        \item Return indicates VDN $\approx$ Q+FIX-mono.
        \item Winrate indicates VDN $\prec$ Q+FIX-mono.
    \end{itemize}
    \item In \Terranx, the return of QPLEX drops significantly around the $9M$ timestep mark, whereas its winrate is able to recover temporarily, indicating that high winrates are achievable even with low returns.
\end{itemize}
Comparing the final performances in \cref{fig:results:smacv2:mean-returns:aggregate,fig:results:smacv2:mean-winrates:aggregate},
\begin{itemize}
    \item Return indicates VDN $\prec$ QMIX $\prec$ QPLEX.
    \item Winrate indicates QPLEX $\prec$ VDN $\approx$ QMIX.
\end{itemize}

\paragraph{Winrate results discussion}

Despite the notable differences between returns and winrates as evaluation metrics, the winrate-based evaluation arrives to largely the same conclusions as the return-based one in the main document, with respect to the performance evaluation of Q+FIX compared to other baselines.

As in the return-based results, VDN fails to be a competitive baseline on its own for most scenarios, likely due to the well-known limited representation.
Fixing VDN via Q+FIX-sum, we are able to overcome this limitation (as noted by the performance gap between VDN and Q+FIX-sum), expanding its representation space and reaching SOTA performance.

As in the return-based results, QMIX sometimes exhibits fast initial learning speeds, albeit often to a sub-competitive final performance (\Protossv, \Terranv, \Terranx, \Zergx, \Terranxx, \Zergxx), again a likely consequence of its limited representation.
Fixing QMIX via Q+FIX-mono, we are often able to exploit the initial learning speeds and complement them with improved performance at convergence reaching SOTA performance.

Compared to return-based results, QPLEX appears less competitive, and performs very well in fewer scenarios (\Protossxx, \Terranxx, \Zergxx), and underperforms in more (\Terranv, \Zergx), and exhibits the same troubling convergence instabilities as well (\Zergv, \Terranx).
Q+FIX-lin, as the simplified variant inspired by QPLEX, manages to avoid such convergence instabilities, plausibly as a consequence of the simpler structure.

As in the return-based results, Q+FIX-sum, Q+FIX-mono, and Q+FIX-lin achieve similar learning performances in most cases, with only minor differences across scenarios.
Compared to the return-based results, it is Q+FIX-mono that may be slightly outperforming other variants in some scenarios (\Terranv, \Zergv).

The aggregate results in \cref{fig:results:smacv2:mean-winrates:aggregate,fig:results:smacv2:iqm-winrates:aggregate} largely confirm the trends discussed above.
Even when employing the IQM measure, which ignores the unstable QPLEX outlier rus, Q+FIX comes out as achieving higher performance.
Despite the concerning difference between the return and winrate metrics, both demonstrate that Q+FIX succeeds in enhancing the native performances of VDN and QMIX fixees, and lifts them to a similar level as QPLEX while maintaining more stable convergence.

\paragraph{Model size ablation}

One of the major appeals of Q+FIX over prior models is in its simplicity, and its ability to enhance prior models to achieve IGM-complete value function decomposition with small models.
Because Q+FIX operates by augmenting existing fixee models with additional models $w(\jh, \ja)$ and $b(\jh)$, there may be other concerns regarding whether the superior performance of Q+FIX comes simply as a consequence of the larger parameterization compared to the corresponding fixee.
\cref{tab:sizes:ablation} contains a complete list of mixer sizes.
Note that the mixer of Q+FIX-sum is always larger than that of VDN, and the mixer of Q+FIX-mono is always larger than that of QMIX.
Therefore, there is a potential concern that the performance of Q+FIX (compared to its corresponding fixee) is driven by the additional parameterization rather than other factors like its proven theoretical properties.

In this section, we present an ablation that disproves this concern by comparing the performance of a \emph{bigger} fixee with a corresponding Q+FIX variant that employs a \emph{smaller} fixee.
We note that this ablation is only possible for the case of QMIX and Q+FIX-mono:
\begin{enumerate*}[label=(\roman*)]
\item VDN has no mixing network;  therefore it is not possible to perform this ablation for VDN and Q+FIX-sum.
\item QPLEX is never used as a fixee;  therefore it is not possible to perform this ablation for QPLEX (also, the Q+FIX models are all significantly smaller than QPLEX to begin with).
\end{enumerate*}
Therefore, we implement a \emph{bigger} variant of QMIX (QMIX-big) and a \emph{smaller} variant of Q+FIX-mono (Q+FIX-mono-small).
See in \cref{tab:sizes:ablation} that the size of Q+FIX-mono-small is now both smaller than that of QMIX-big, and more comparable to those of Q+FIX-sum and Q+FIX-lin.

\cref{fig:results:smacv2:ablation} shows the results of this ablation evaluation; to focus on the matter at hand, we only show the relevant performance of QMIX and Q+FIX-mono methods.
As can be seen, the performance of Q+FIX-mono-small is analogous to that of +FIX-mono, and the performance of QMIX-big is analogous to that of QMIX.
These results strongly confirm that the superior performance of Q+FIX-mono is not caused by the larger parameterization, but by our proposed fixing structure.

\begin{table}
    \centering
    \caption{
    SMACv2 mixer sizes in number of parameters.
    Smallest (non-zero) models highlighted.
    }
    \label{tab:sizes:ablation}
    
    \begin{tabular}{lrrrrrr}
        & \multicolumn{3}{c}{\Protoss}
        & \multicolumn{3}{c}{\Terran, \Zerg} \\
        & \sizev & \sizex & \sizexx
        & \sizev & \sizex & \sizexx \\
        \toprule
        VDN
            & 0 k & 0 k & 0 k
            & 0 k & 0 k & 0 k \\
        \midrule
        QMIX
            & 38 k & 83 k & 201 k
            & 36 k & 79 k & 194 k \\
        QPLEX
            & 135 k & 326 k & 882 k
            & 126 k & 308 k & 846 k \\
        \rowcolor{gray!20}
        Q+FIX-sum
            & 20 k & 50 k & 138 k
            & 19 k & 48 k & 133 k \\
        Q+FIX-mono
            & 54 k & 180 k & 743 k
            & 50 k & 169 k & 708 k \\
        \rowcolor{gray!20}
        Q+FIX-lin
            & 21 k & 51 k & 140 k
            & 19 k & 48 k & 135 k \\
        \midrule
        QMIX-big
            & 166 k & 341 k & 767 k
            & 161 k & 331 k & 747 k \\
        Q+FIX-mono-small
            & 29 k & 83 k & 290 k
            & 27 k & 78 k & 277 k \\
        \bottomrule
    \end{tabular}
\end{table}

\begin{figure}
  \centering

  \begin{subfigure}{\linewidth}
    \centering
    \includegraphics[width=\linewidth]{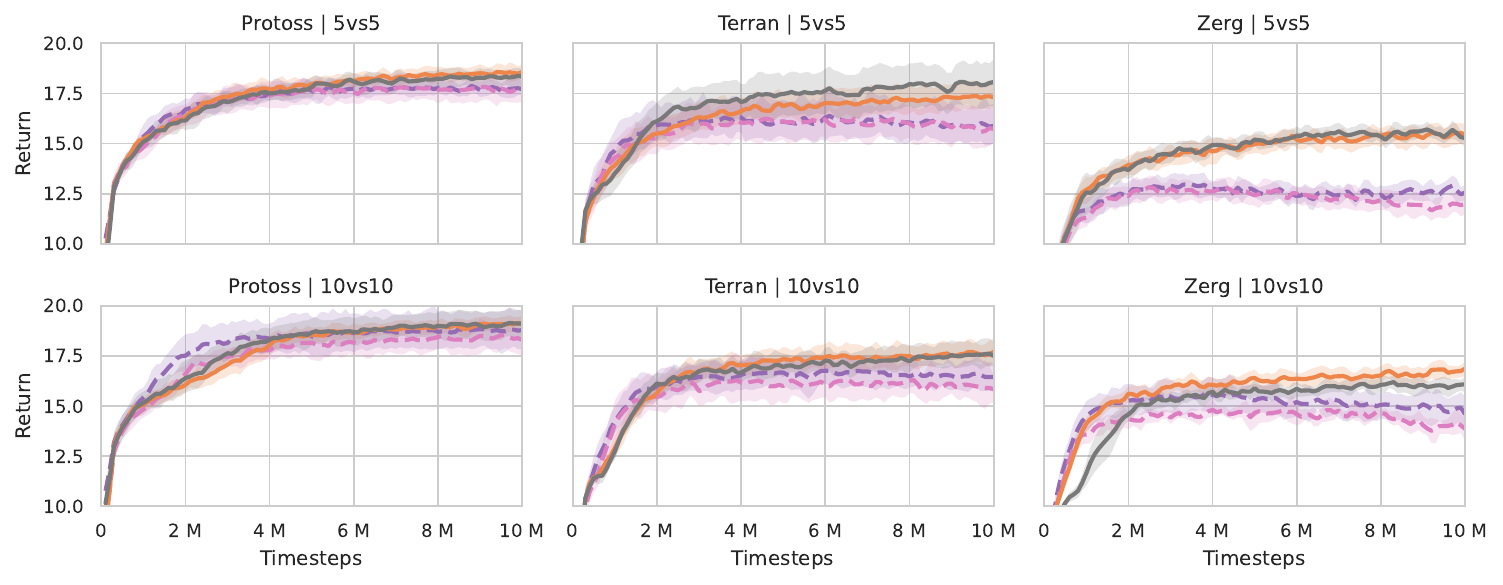}
    \caption{SMACv2 return mean ($5$ seeds).}
  \end{subfigure}

  \begin{subfigure}{0.44\linewidth}
    \centering
    \includegraphics[height=2.7cm]{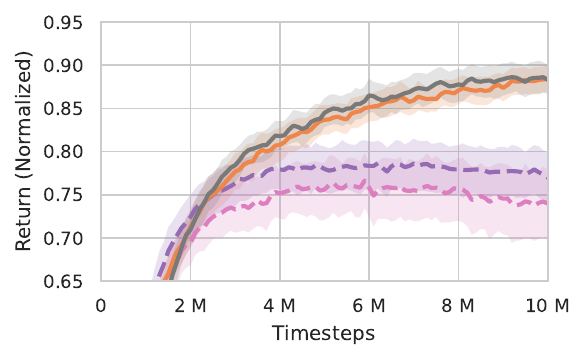}
    \caption{SMACv2 return mean aggregate ($30$ seeds).}
  \end{subfigure}
  \begin{subfigure}{0.55\linewidth}
    \includegraphics[height=2.7cm]{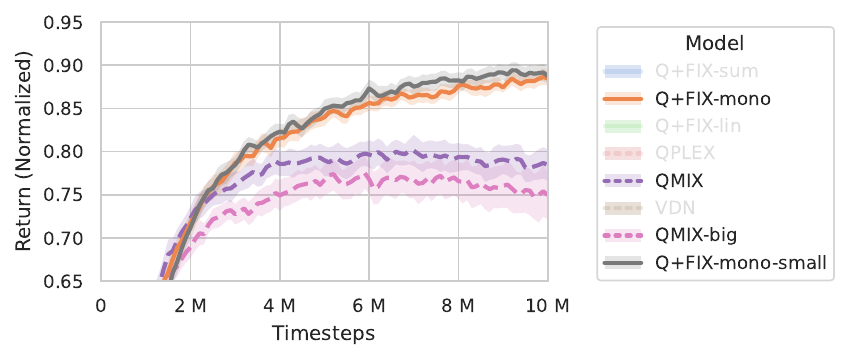}
    \caption{SMACv2 return IQM aggregate ($30$ seeds).}
  \end{subfigure}

  \caption{SMACv2 ablation results, bootstrapped $95\%$.
  Aggregation computed as in \cref{fig:results:smacv2:returns}.}
  \label{fig:results:smacv2:ablation}
\end{figure}

\paragraph{Probability of improvement}

\citet{agarwal_deep_2021} also suggest the use of \emph{probability of improvement} (POI) as a criterion for evaluation that is resilient to data outliers.
This metric measures the likelihood that a random run based on one method outperforms a random run based on another method, while ignoring the size of the performance gap.
If method $X$ has been evaluated empirically $N$ times with performances $\hat X = \{\hat x_i \}_{i=1}^N$, and method $Y$ has been evaluated empirically $M$ times with performances $\hat Y = \{\hat y_i \}_{i=1}^M$, we estimate the POI as
\begin{equation}
  \Pr(X > Y) \approx \frac{1}{N \cdot M} \sum_{\hat x\in \hat X, \hat y\in\hat Y} \Ind\left[ \hat x > \hat y \right] \,.
\end{equation}
In their work, \citet{agarwal_deep_2021} demonstrate this criterion assuming that each run is summarized by a single scalar (e.g., final performance); since we are both concerned with learning speed and are uncertain how to fairly pick a single scalar performance for each run, we instead perform this calculation over the entire learning phase.

\cref{fig:results:smacv2:poi} contains our aggregate POI results for SMACv2.
\citet{agarwal_deep_2021} note that a POI that is above $50\%$ with its entire CI indicates a statistically significant result; out of all methods, Q+FIX-sum is the only one to achieve this against all other methods.

\begin{figure}
  \centering

  \begin{subfigure}{0.49\linewidth}
    \centering
    \includegraphics[width=\linewidth]{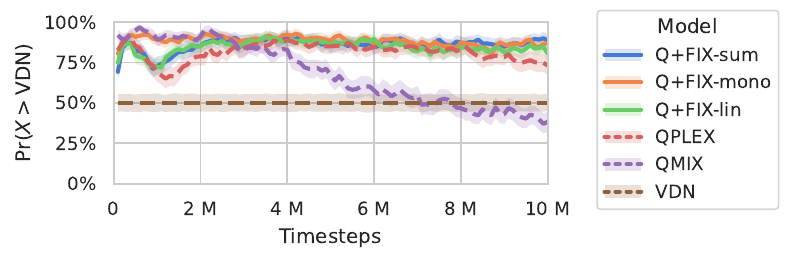}
    \caption{POI of model ``X'' over VDN.}
  \end{subfigure}
  \begin{subfigure}{0.49\linewidth}
    \centering
    \includegraphics[width=\linewidth]{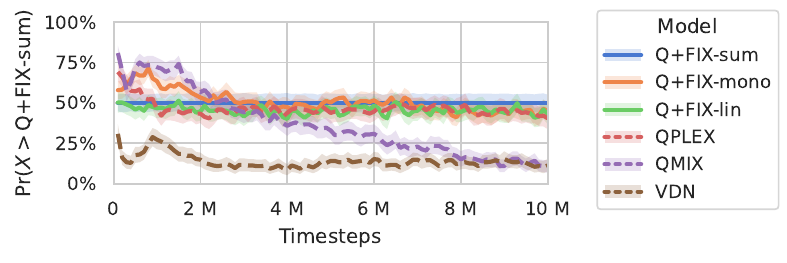}
    \caption{POI of model ``X'' over Q+FIX-sum.}
  \end{subfigure}

  \begin{subfigure}{0.49\linewidth}
    \centering
    \includegraphics[width=\linewidth]{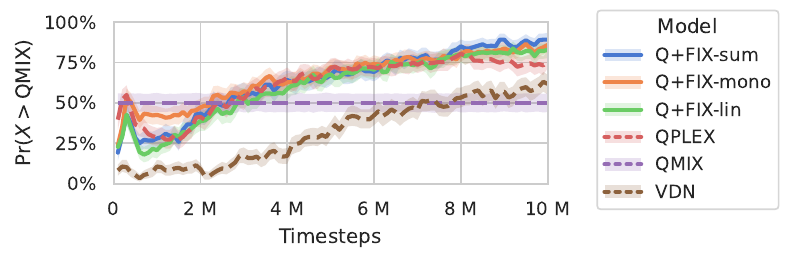}
    \caption{POI of model ``X'' over QMIX.}
  \end{subfigure}
  \begin{subfigure}{0.49\linewidth}
    \centering
    \includegraphics[width=\linewidth]{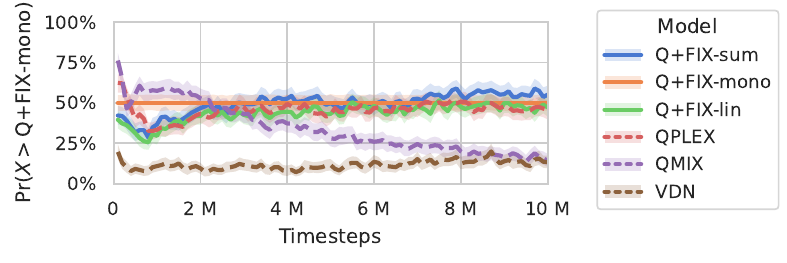}
    \caption{POI of model ``X'' over Q+FIX-mono.}
  \end{subfigure}

  \begin{subfigure}{0.49\linewidth}
    \centering
    \includegraphics[width=\linewidth]{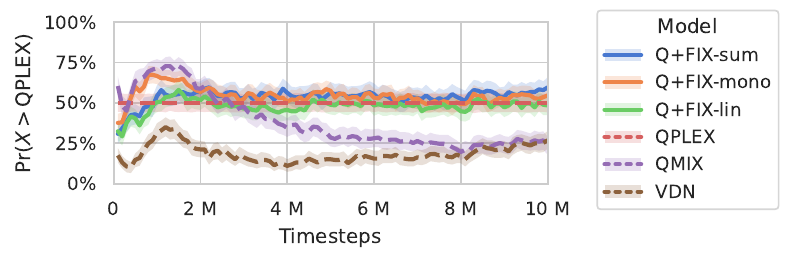}
    \caption{POI of model ``X'' over QPLEX.}
  \end{subfigure}
  \begin{subfigure}{0.49\linewidth}
    \centering
    \includegraphics[width=\linewidth]{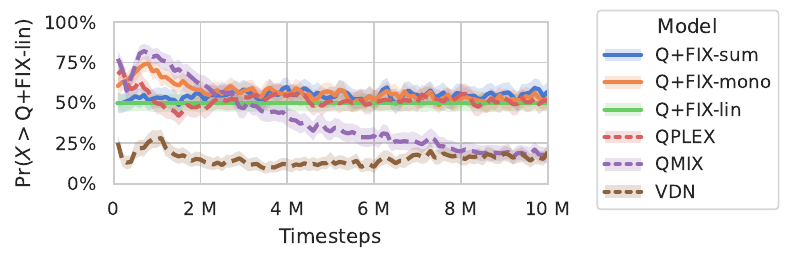}
    \caption{POI of model ``X'' over Q+FIX-lin.}
  \end{subfigure}

  \caption{Aggregate probability of improvement (POI), bootstrapped 95\% CI.}
  \label{fig:results:smacv2:poi}
\end{figure}

\subsection{Overcooked}
\label{sec:additional-results:overcooked}

\paragraph{Observability}

Overcooked is a fully observable environment, with each agent receiving observations whose information content is equivalent to the state.
Therefore, the challenge of these tasks is primarily one of coordination and subtask assignment over information gathering.
The state is provided as a tensor with shape $H\times W \times C$, with $C=26$ (mostly but not exclusively binary) channels encoding agent positions and orientations, and positions of tables, pots, plates, various ingredients, etc.

\paragraph{Coordination}

Notably, the tasks in overcooked do not strictly require tight coordination between agents.
Though some tasks may need both agents to contribute in different ways to the same plate being completed, that cooperation is not under strict coordination requirements.
Though the agents may achieve higher efficiency and performance if they coordinate optimally, the tasks can be completed even if the agents act relatively independently.
We believe this can explain some of the results in our evaluation, especially in terms of the relatively good performance of methods like VDN that hardly enforce strong coordination.

\paragraph{Implementation details}

For these Overcooked experiments, we found it useful for Q+FIX to employ both \emph{advantage detaching} and \emph{intervention annealing} with $\lambda$ descending linearly from $1$ to $0$ over the first $500k$ timesteps (10\% of training).

\paragraph{Additional results}

\cref{fig:results:more-overcooked:returns} shows the results for all five evaluated layouts: \CrampedRoom, \AsymmetricAdvantages, \CoordinationRing, \ForcedCoordination, and \CounterCircuit.
The tasks are categorically different and not directly comparable, but there is a progression in difficulty from the first to the last.
\CoordinationRing\ is a simple task solved adequately by all methods.
Performances start to differentiate more strongly in the other layout.
Notably, QMIX has some trouble in these layouts, even compared to simpler methods like VDN.
We believe this may be due to the coordination properties of these layouts (as described previously in this section), which may benefit simpler methods like VDN and independent learners.
Nonetheless, in all scenarios, Q+FIX variants are the best performing, achieving statistically significant performance improvements compared to the baselines in four of the five layouts.

\begin{figure}
  \centering
  \includegraphics[width=\linewidth]{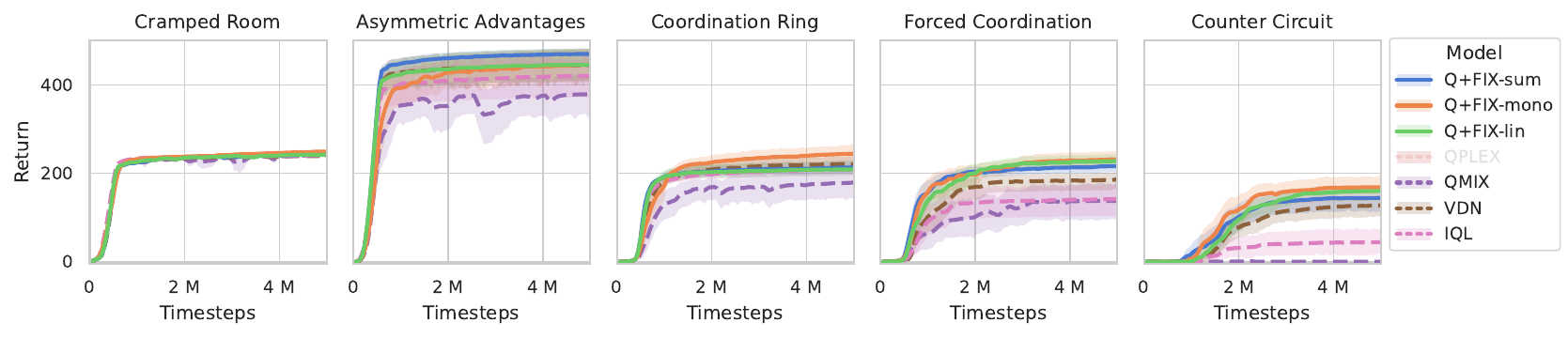}
  \caption{Overcooked return mean, bootstrapped $95\%$ CI ($20$ seeds).}
  \label{fig:results:more-overcooked:returns}
\end{figure}

\section{Architectures and hyperparameters}
\label{sec:architectures}

\subsection{SMACv2}
\label{sec:architectures:smacv2}

Baseline methods are used and run as implemented in the \pymarl\ repository\footnote{\url{https://github.com/benellis3/pymarl2}}, using the pre-optimized hyperparameters as provided by the corresponding configs.
Q+FIX methods are implemented to match the baseline implementations completely, with the only difference being the mixer type and architecture, and are run using the same hyperparameters as the baselines.
All implementations use the Adam optimizer~\cite{kingma_adam_2017}.

Due to their complex nature (including the use of hypernetworks and attention modules) we omit a full description of the mixing architectures for QMIX and QPLEX.
We refer the reader to the corresponding publications and \pymarl\ implementations\footnote{\url{https://github.com/benellis3/pymarl2/blob/master/src/modules/mixers/qmix.py}}\footnote{\url{https://github.com/benellis3/pymarl2/blob/master/src/modules/mixers/dmaq_general.py}}.

\paragraph{Agent Model $\qmodel_i(h_i, a_i)$}

All methods employ the same architecture to compute the individual utilities $\qmodel_i(h_i, a_i)$.
As SMAXv2 is partially-observable and provides observations directly as feature vectors, this architecture employs the following layers:
\begin{itemize}
    \item \textbf{Inputs:}
        \begin{itemize}
            \item Observation vector $\realset^d$ ($d$ variable per scenario).
            \item Agent ID one-hot encoding $\{0, 1\}^N$.
        \end{itemize}
    \item \textbf{Layers:}
        \begin{itemize}
            \item \texttt{Linear(output\_dim=64)} + \texttt{ReLU()}
            \item \texttt{GRUCell(output\_dim=64)}
            \item \texttt{Linear(output\_dim=\#actions)}
        \end{itemize}
\item \textbf{Output:} Action values $\realset^{|\aset_i|}$, one per action.
\end{itemize}

%
%
%
%
%
%
%
%


\paragraph{Q+FIX Weight Model $w(s, \ja)$}
\begin{itemize}
\item \textbf{Input:}
    \begin{itemize}
        \item State vector $\realset^d$ ($d$ variable per scenario).
        \item Agent actions one-hot encodings $\{0, 1\}^{\sum_i |\aset_i|}$.
    \end{itemize}
\item \textbf{Layers:}
    \begin{itemize}
        \item \texttt{Linear(output\_dim=64)} + \texttt{ReLU()}
        \item \texttt{Linear(output\_dim=1)} (if Q+FIX-\{sum,mono\}) \\
            \texttt{Linear(output\_dim=N)} (if Q+FIX-lin)
        \item \texttt{lambda w: |w+1|-1+10e-8}
    \end{itemize}
\item \textbf{Outputs:} Weights $w(s, \ja) \in\realset_{>-1}$ (if Q+FIX-\{sum,mono\}) \\
    \phantom{\textbf{Outputs:}} Weights $w(s, \ja) \in\realset_{>-1}^N$ (if Q+FIX-lin).
\end{itemize}

\paragraph{Q+FIX Bias Model $b(s)$}
\begin{itemize}
\item \textbf{Input:} State vector $\realset^d$ ($d$ variable per scenario).
\item \textbf{Layers:}
    \begin{itemize}
        \item \texttt{Linear(output\_dim=64)} + \texttt{ReLU()}
        \item \texttt{Linear(output\_dim=1)}
    \end{itemize}
\item \textbf{Output:} Bias $b(s) \in\realset$.
\end{itemize}


\subsection{Overcooked}
\label{sec:architectures:overcooked}

Baseline methods are used and run as implemented in the \jaxmarl\ repository\footnote{\url{https://github.com/FLAIROx/JaxMARL}}, using the pre-optimized hyperparameters as provided by the corresponding configs.
Q+FIX methods are implemented to match the baseline implementations completely, with the only difference being the mixer type and architecture, and are run using the same hyperparameters as the baselines.
All implementations use the \emph{rectified} Adam (RAdam) optimizer~\cite{liu_variance_2019}.

\paragraph{Agent Model $\qmodel_i(h_i, a_i)$}

All methods employ the same architecture to compute the individual utilities $\qmodel_i(h_i, a_i)$.
As Overcooked is fully-observable and provides states as a grid (tensor) of categorical data, this architecture employs the following layers:
\begin{itemize}
\item \textbf{Input:} State grid $\naturalset^{H\times W\times C}$ ($C=26$ channels, mostly binary).
\item \textbf{Layers:}
    \begin{itemize}
        \item \texttt{Conv(output\_dim=32, kernel\_size=(5, 5))} + \texttt{ReLU()}
        \item \texttt{Conv(output\_dim=32, kernel\_size=(3, 3))} + \texttt{ReLU()}
        \item \texttt{Conv(output\_dim=32, kernel\_size=(3, 3))} + \texttt{ReLU()} + \texttt{Flatten()}
        \item \texttt{Linear(output\_dim=64)} + \texttt{ReLU()}
        \item \texttt{Linear(output\_dim=64)} + \texttt{ReLU()}
        \item \texttt{Linear(output\_dim=\#actions)}
    \end{itemize}
\item \textbf{Output:} Action values $\realset^{|\aset_i|}$, one per action.
\end{itemize}

\paragraph{Q+FIX Weight Models $w(s, \ja)$}
\begin{itemize}
\item \textbf{Input:} 
    \begin{itemize}
        \item State grid $\naturalset^{H\times W\times C}$ ($C=26$ channels, mostly binary).
        \item Agent actions one-hot encodings $\{0, 1\}^{\sum_i |\aset_i|}$.
    \end{itemize}
\item \textbf{Layers:}
    \begin{itemize}
        \item \texttt{Conv(output\_dim=64, kernel\_size=(5, 5))} + \texttt{ReLU()}
        \item \texttt{Conv(output\_dim=64, kernel\_size=(3, 3))} + \texttt{ReLU()} + \texttt{Flatten()}
        \item \texttt{Linear(output\_dim=64)} + \texttt{ReLU()}
        \item \texttt{Linear(output\_dim=1)} (if Q+FIX-\{sum,mono\}) \\
            \texttt{Linear(output\_dim=N)} (if Q+FIX-lin)
        \item \texttt{lambda w: |w+1|-1+10e-8}
    \end{itemize}
\item \textbf{Outputs:} Weights $w(s, \ja) \in\realset_{>-1}$ (if Q+FIX-\{sum,mono\}) \\
    \phantom{\textbf{Outputs:}} Weights $w(s, \ja) \in\realset_{>-1}^N$ (if Q+FIX-lin).
\end{itemize}

\paragraph{Q+FIX Bias Model $b(s)$}
\begin{itemize}
\item \textbf{Input:} State grid $\naturalset^{H\times W\times C}$ ($C=26$ channels, mostly binary).
\item \textbf{Layers:}
    \begin{itemize}
        \item \texttt{Linear(output\_dim=64)} + \texttt{ReLU()}
        \item \texttt{Linear(output\_dim=1)}
    \end{itemize}
\item \textbf{Output:} Bias $b(s) \in\realset$.
\end{itemize}



\section{Experiments compute resources}
\label{sec:resources}

Experiments were distributed (unevenly) primarily across two workstations:
\begin{itemize}
\item \textbf{Type:} Standalone workstation, \\
      \textbf{CPU:} Intel(R) Core(TM) i7-7700K CPU @ 4.20GHz, \\
      \textbf{GPU(s):} 2x NVIDIA GeForce GTX 1080.
\item \textbf{Type:} Standalone workstation, \\
      \textbf{CPU:} AMD Ryzen Threadripper 7960X 24-Cores, \\
      \textbf{GPU(s):} 1x NVIDIA GeForce RTX 4090.
\end{itemize}

The time of executing a single run can differ greatly depending on the workstation, the environment, the method, and model size.
The following is only a very rough estimate of total sequential runtime:
\begin{description}
    \item[SMACv2] \pymarl\ implementations can be very slow due to the CPU-bound environment, and vary somewhere between $\SI{5}{\hour}$ and $\SI{20}{\hour}$ per run.
        For the main results, since we execute $6$ methods for $5$ runs in $9$ scenarios, which amounts to $6 \cdot 5 \cdot 9 = 270$ independent runs and roughly $270 \cdot \SI{12}{\hour} = \SI{135}{\day}$ of sequential runtime.
        For the ablation results, we execute an additional $2$ methods for $5$ runs in $6$ scenarios, which amounts to $2 \cdot 5 \cdot 6 = 60$ independent runs and roughly $60 \cdot \SI{12}{\hour} = \SI{30}{\day}$ of sequential runtime.
    \item[Overcooked] \jaxmarl\ implementations are much faster, and vary between $\SI{15}{\minute}$ and $\SI{60}{\minute}$.
        Since we execute $6$ methods for $20$ runs in $5$ layouts, which amounts to $6 \cdot 20 \cdot 5 = 600$ independent runs and roughly $600 \cdot \SI{40}{\minute} = \SI{24000}{\minute} \approx \SI{16}{\day}$ of sequential runtime.
\end{description}

Naturally, the experiments were not executed purely sequentially; however, they still took multiple weeks to complete as a whole, on our available hardware.


\newpage
\section*{NeurIPS Paper Checklist}

\begin{enumerate}

\item {\bf Claims}
    \item[] Question: Do the main claims made in the abstract and introduction accurately reflect the paper's contributions and scope?
    \item[] Answer: \answerYes{} 
    \item[] Justification: The primary contributions (derivation of a simple IGM value decomposition framework, of the QFIX family, and the corresponding evaluation) are all discussed in the main document.
    \item[] Guidelines:
    \begin{itemize}
        \item The answer NA means that the abstract and introduction do not include the claims made in the paper.
        \item The abstract and/or introduction should clearly state the claims made, including the contributions made in the paper and important assumptions and limitations. A No or NA answer to this question will not be perceived well by the reviewers. 
        \item The claims made should match theoretical and experimental results, and reflect how much the results can be expected to generalize to other settings. 
        \item It is fine to include aspirational goals as motivation as long as it is clear that these goals are not attained by the paper. 
    \end{itemize}

\item {\bf Limitations}
    \item[] Question: Does the paper discuss the limitations of the work performed by the authors?
    \item[] Answer: \answerYes{} 
    \item[] Justification: Limitations are mentioned in the conclusions section.
    \item[] Guidelines:
    \begin{itemize}
        \item The answer NA means that the paper has no limitation while the answer No means that the paper has limitations, but those are not discussed in the paper. 
        \item The authors are encouraged to create a separate "Limitations" section in their paper.
        \item The paper should point out any strong assumptions and how robust the results are to violations of these assumptions (e.g., independence assumptions, noiseless settings, model well-specification, asymptotic approximations only holding locally). The authors should reflect on how these assumptions might be violated in practice and what the implications would be.
        \item The authors should reflect on the scope of the claims made, e.g., if the approach was only tested on a few datasets or with a few runs. In general, empirical results often depend on implicit assumptions, which should be articulated.
        \item The authors should reflect on the factors that influence the performance of the approach. For example, a facial recognition algorithm may perform poorly when image resolution is low or images are taken in low lighting. Or a speech-to-text system might not be used reliably to provide closed captions for online lectures because it fails to handle technical jargon.
        \item The authors should discuss the computational efficiency of the proposed algorithms and how they scale with dataset size.
        \item If applicable, the authors should discuss possible limitations of their approach to address problems of privacy and fairness.
        \item While the authors might fear that complete honesty about limitations might be used by reviewers as grounds for rejection, a worse outcome might be that reviewers discover limitations that aren't acknowledged in the paper. The authors should use their best judgment and recognize that individual actions in favor of transparency play an important role in developing norms that preserve the integrity of the community. Reviewers will be specifically instructed to not penalize honesty concerning limitations.
    \end{itemize}

\item {\bf Theory assumptions and proofs}
    \item[] Question: For each theoretical result, does the paper provide the full set of assumptions and a complete (and correct) proof?
    \item[] Answer: \answerYes{}
    \item[] Justification: Every novel proposition/theorem is clearly indicated as such, with clearly stated assumptions, and with a corresponding proof in the appendix.
    Proof sketches were omitted in the main document due to the proofs being strictly technical, and space limitations.
    \item[] Guidelines:
    \begin{itemize}
        \item The answer NA means that the paper does not include theoretical results. 
        \item All the theorems, formulas, and proofs in the paper should be numbered and cross-referenced.
        \item All assumptions should be clearly stated or referenced in the statement of any theorems.
        \item The proofs can either appear in the main paper or the supplemental material, but if they appear in the supplemental material, the authors are encouraged to provide a short proof sketch to provide intuition. 
        \item Inversely, any informal proof provided in the core of the paper should be complemented by formal proofs provided in appendix or supplemental material.
        \item Theorems and Lemmas that the proof relies upon should be properly referenced. 
    \end{itemize}

    \item {\bf Experimental result reproducibility}
    \item[] Question: Does the paper fully disclose all the information needed to reproduce the main experimental results of the paper to the extent that it affects the main claims and/or conclusions of the paper (regardless of whether the code and data are provided or not)?
    \item[] Answer: \answerYes{}
    \item[] Justification: We mention important implementation details such as the \emph{detaching of advantages} and the \emph{annealing of the intervention} in the main submission, relevant model descriptions in the appendix, and will link to the corresponding code bases with instructions to run the experiments for the camera ready.
    \item[] Guidelines:
    \begin{itemize}
        \item The answer NA means that the paper does not include experiments.
        \item If the paper includes experiments, a No answer to this question will not be perceived well by the reviewers: Making the paper reproducible is important, regardless of whether the code and data are provided or not.
        \item If the contribution is a dataset and/or model, the authors should describe the steps taken to make their results reproducible or verifiable. 
        \item Depending on the contribution, reproducibility can be accomplished in various ways. For example, if the contribution is a novel architecture, describing the architecture fully might suffice, or if the contribution is a specific model and empirical evaluation, it may be necessary to either make it possible for others to replicate the model with the same dataset, or provide access to the model. In general. releasing code and data is often one good way to accomplish this, but reproducibility can also be provided via detailed instructions for how to replicate the results, access to a hosted model (e.g., in the case of a large language model), releasing of a model checkpoint, or other means that are appropriate to the research performed.
        \item While NeurIPS does not require releasing code, the conference does require all submissions to provide some reasonable avenue for reproducibility, which may depend on the nature of the contribution. For example
        \begin{enumerate}
            \item If the contribution is primarily a new algorithm, the paper should make it clear how to reproduce that algorithm.
            \item If the contribution is primarily a new model architecture, the paper should describe the architecture clearly and fully.
            \item If the contribution is a new model (e.g., a large language model), then there should either be a way to access this model for reproducing the results or a way to reproduce the model (e.g., with an open-source dataset or instructions for how to construct the dataset).
            \item We recognize that reproducibility may be tricky in some cases, in which case authors are welcome to describe the particular way they provide for reproducibility. In the case of closed-source models, it may be that access to the model is limited in some way (e.g., to registered users), but it should be possible for other researchers to have some path to reproducing or verifying the results.
        \end{enumerate}
    \end{itemize}

\item {\bf Open access to data and code}
    \item[] Question: Does the paper provide open access to the data and code, with sufficient instructions to faithfully reproduce the main experimental results, as described in supplemental material?
    \item[] Answer: \answerYes{}
    \item[] Justification: Anonymized code is provided for both evaluation frameworks, \pymarl\ (\url{https://anonymous.4open.science/r/pymarl2-C5EF/}) and \jaxmarl\ (\url{https://anonymous.4open.science/r/JaxMARL-682A/}).
    Instructions for the \pymarl\ implementation are provided in the readme.
    Instructions for the \jaxmarl\ implementation follow those from the original repo.
    \item[] Guidelines:
    \begin{itemize}
        \item The answer NA means that paper does not include experiments requiring code.
        \item Please see the NeurIPS code and data submission guidelines (\url{https://nips.cc/public/guides/CodeSubmissionPolicy}) for more details.
        \item While we encourage the release of code and data, we understand that this might not be possible, so “No” is an acceptable answer. Papers cannot be rejected simply for not including code, unless this is central to the contribution (e.g., for a new open-source benchmark).
        \item The instructions should contain the exact command and environment needed to run to reproduce the results. See the NeurIPS code and data submission guidelines (\url{https://nips.cc/public/guides/CodeSubmissionPolicy}) for more details.
        \item The authors should provide instructions on data access and preparation, including how to access the raw data, preprocessed data, intermediate data, and generated data, etc.
        \item The authors should provide scripts to reproduce all experimental results for the new proposed method and baselines. If only a subset of experiments are reproducible, they should state which ones are omitted from the script and why.
        \item At submission time, to preserve anonymity, the authors should release anonymized versions (if applicable).
        \item Providing as much information as possible in supplemental material (appended to the paper) is recommended, but including URLs to data and code is permitted.
    \end{itemize}

\item {\bf Experimental setting/details}
    \item[] Question: Does the paper specify all the training and test details (e.g., data splits, hyperparameters, how they were chosen, type of optimizer, etc.) necessary to understand the results?
    \item[] Answer: \answerYes{}
    \item[] Justification: As the experiments are based on well established frameworks \pymarl\ and \jaxmarl, most experimental settings and details are provided by the corresponding codebases.
    Any additional component (e.g., the architectures of Q+FIX) is both described in the appendix, provided as supplementary material, and will be linked in the camera ready).
    \item[] Guidelines:
    \begin{itemize}
        \item The answer NA means that the paper does not include experiments.
        \item The experimental setting should be presented in the core of the paper to a level of detail that is necessary to appreciate the results and make sense of them.
        \item The full details can be provided either with the code, in appendix, or as supplemental material.
    \end{itemize}

\item {\bf Experiment statistical significance}
    \item[] Question: Does the paper report error bars suitably and correctly defined or other appropriate information about the statistical significance of the experiments?
    \item[] Answer: \answerYes{}
    \item[] Justification: All results are shown with clearly-stated bootstrapped $95\%$ confidence interval.
    \item[] Guidelines:
    \begin{itemize}
        \item The answer NA means that the paper does not include experiments.
        \item The authors should answer "Yes" if the results are accompanied by error bars, confidence intervals, or statistical significance tests, at least for the experiments that support the main claims of the paper.
        \item The factors of variability that the error bars are capturing should be clearly stated (for example, train/test split, initialization, random drawing of some parameter, or overall run with given experimental conditions).
        \item The method for calculating the error bars should be explained (closed form formula, call to a library function, bootstrap, etc.)
        \item The assumptions made should be given (e.g., Normally distributed errors).
        \item It should be clear whether the error bar is the standard deviation or the standard error of the mean.
        \item It is OK to report 1-sigma error bars, but one should state it. The authors should preferably report a 2-sigma error bar than state that they have a 96\% CI, if the hypothesis of Normality of errors is not verified.
        \item For asymmetric distributions, the authors should be careful not to show in tables or figures symmetric error bars that would yield results that are out of range (e.g. negative error rates).
        \item If error bars are reported in tables or plots, The authors should explain in the text how they were calculated and reference the corresponding figures or tables in the text.
    \end{itemize}

\item {\bf Experiments compute resources}
    \item[] Question: For each experiment, does the paper provide sufficient information on the computer resources (type of compute workers, memory, time of execution) needed to reproduce the experiments?
    \item[] Answer: \answerYes{}
    \item[] Justification: The appendix contains a section describing the hardware used to run the experiments, and rough calculations for the total time of sequential execution of all experiments.
    \item[] Guidelines:
    \begin{itemize}
        \item The answer NA means that the paper does not include experiments.
        \item The paper should indicate the type of compute workers CPU or GPU, internal cluster, or cloud provider, including relevant memory and storage.
        \item The paper should provide the amount of compute required for each of the individual experimental runs as well as estimate the total compute. 
        \item The paper should disclose whether the full research project required more compute than the experiments reported in the paper (e.g., preliminary or failed experiments that didn't make it into the paper). 
    \end{itemize}
    
\item {\bf Code of ethics}
    \item[] Question: Does the research conducted in the paper conform, in every respect, with the NeurIPS Code of Ethics \url{https://neurips.cc/public/EthicsGuidelines}?
    \item[] Answer: \answerYes{}
    \item[] Justification: The work does not violate any of the guidelines outlines in the Code of Ethics.
    \item[] Guidelines:
    \begin{itemize}
        \item The answer NA means that the authors have not reviewed the NeurIPS Code of Ethics.
        \item If the authors answer No, they should explain the special circumstances that require a deviation from the Code of Ethics.
        \item The authors should make sure to preserve anonymity (e.g., if there is a special consideration due to laws or regulations in their jurisdiction).
    \end{itemize}

\item {\bf Broader impacts}
    \item[] Question: Does the paper discuss both potential positive societal impacts and negative societal impacts of the work performed?
    \item[] Answer: \answerNA{}
    \item[] Justification: The work develops general-purpose foundational methods not intrinsically tied to particular applications or deployments.
    \item[] Guidelines:
    \begin{itemize}
        \item The answer NA means that there is no societal impact of the work performed.
        \item If the authors answer NA or No, they should explain why their work has no societal impact or why the paper does not address societal impact.
        \item Examples of negative societal impacts include potential malicious or unintended uses (e.g., disinformation, generating fake profiles, surveillance), fairness considerations (e.g., deployment of technologies that could make decisions that unfairly impact specific groups), privacy considerations, and security considerations.
        \item The conference expects that many papers will be foundational research and not tied to particular applications, let alone deployments. However, if there is a direct path to any negative applications, the authors should point it out. For example, it is legitimate to point out that an improvement in the quality of generative models could be used to generate deepfakes for disinformation. On the other hand, it is not needed to point out that a generic algorithm for optimizing neural networks could enable people to train models that generate Deepfakes faster.
        \item The authors should consider possible harms that could arise when the technology is being used as intended and functioning correctly, harms that could arise when the technology is being used as intended but gives incorrect results, and harms following from (intentional or unintentional) misuse of the technology.
        \item If there are negative societal impacts, the authors could also discuss possible mitigation strategies (e.g., gated release of models, providing defenses in addition to attacks, mechanisms for monitoring misuse, mechanisms to monitor how a system learns from feedback over time, improving the efficiency and accessibility of ML).
    \end{itemize}
    
\item {\bf Safeguards}
    \item[] Question: Does the paper describe safeguards that have been put in place for responsible release of data or models that have a high risk for misuse (e.g., pretrained language models, image generators, or scraped datasets)?
    \item[] Answer: \answerNA{}
    \item[] Justification: The paper poses no such risks.
    \item[] Guidelines:
    \begin{itemize}
        \item The answer NA means that the paper poses no such risks.
        \item Released models that have a high risk for misuse or dual-use should be released with necessary safeguards to allow for controlled use of the model, for example by requiring that users adhere to usage guidelines or restrictions to access the model or implementing safety filters. 
        \item Datasets that have been scraped from the Internet could pose safety risks. The authors should describe how they avoided releasing unsafe images.
        \item We recognize that providing effective safeguards is challenging, and many papers do not require this, but we encourage authors to take this into account and make a best faith effort.
    \end{itemize}

\item {\bf Licenses for existing assets}
    \item[] Question: Are the creators or original owners of assets (e.g., code, data, models), used in the paper, properly credited and are the license and terms of use explicitly mentioned and properly respected?
    \item[] Answer: \answerYes{}
    \item[] Justification: The paper's experiments are based on two well-known open-source framework, \pymarl\ and \jaxmarl, both of which are provided with the permissive Apache License 2.0.
    Our own implementations continue to use the same license.
    \item[] Guidelines:
    \begin{itemize}
        \item The answer NA means that the paper does not use existing assets.
        \item The authors should cite the original paper that produced the code package or dataset.
        \item The authors should state which version of the asset is used and, if possible, include a URL.
        \item The name of the license (e.g., CC-BY 4.0) should be included for each asset.
        \item For scraped data from a particular source (e.g., website), the copyright and terms of service of that source should be provided.
        \item If assets are released, the license, copyright information, and terms of use in the package should be provided. For popular datasets, \url{paperswithcode.com/datasets} has curated licenses for some datasets. Their licensing guide can help determine the license of a dataset.
        \item For existing datasets that are re-packaged, both the original license and the license of the derived asset (if it has changed) should be provided.
        \item If this information is not available online, the authors are encouraged to reach out to the asset's creators.
    \end{itemize}

\item {\bf New assets}
    \item[] Question: Are new assets introduced in the paper well documented and is the documentation provided alongside the assets?
    \item[] Answer: \answerYes{}
    \item[] Justification: The paper provides implementations of Q+FIX based on two well-known multi-agent RL frameworks: \pymarl\ and \jaxmarl.
    These will be provided as forks from the corresponding repositories.
    \item[] Guidelines:
    \begin{itemize}
        \item The answer NA means that the paper does not release new assets.
        \item Researchers should communicate the details of the dataset/code/model as part of their submissions via structured templates. This includes details about training, license, limitations, etc. 
        \item The paper should discuss whether and how consent was obtained from people whose asset is used.
        \item At submission time, remember to anonymize your assets (if applicable). You can either create an anonymized URL or include an anonymized zip file.
    \end{itemize}

\item {\bf Crowdsourcing and research with human subjects}
    \item[] Question: For crowdsourcing experiments and research with human subjects, does the paper include the full text of instructions given to participants and screenshots, if applicable, as well as details about compensation (if any)? 
    \item[] Answer: \answerNA{}
    \item[] Justification: The paper does not involve crowdsourcing nor research with human subjects.
    \item[] Guidelines:
    \begin{itemize}
        \item The answer NA means that the paper does not involve crowdsourcing nor research with human subjects.
        \item Including this information in the supplemental material is fine, but if the main contribution of the paper involves human subjects, then as much detail as possible should be included in the main paper. 
        \item According to the NeurIPS Code of Ethics, workers involved in data collection, curation, or other labor should be paid at least the minimum wage in the country of the data collector. 
    \end{itemize}

\item {\bf Institutional review board (IRB) approvals or equivalent for research with human subjects}
    \item[] Question: Does the paper describe potential risks incurred by study participants, whether such risks were disclosed to the subjects, and whether Institutional Review Board (IRB) approvals (or an equivalent approval/review based on the requirements of your country or institution) were obtained?
    \item[] Answer: \answerNA{}
    \item[] Justification: The paper does not involve crowdsourcing nor research with human subjects.
    \item[] Guidelines:
    \begin{itemize}
        \item The answer NA means that the paper does not involve crowdsourcing nor research with human subjects.
        \item Depending on the country in which research is conducted, IRB approval (or equivalent) may be required for any human subjects research. If you obtained IRB approval, you should clearly state this in the paper. 
        \item We recognize that the procedures for this may vary significantly between institutions and locations, and we expect authors to adhere to the NeurIPS Code of Ethics and the guidelines for their institution. 
        \item For initial submissions, do not include any information that would break anonymity (if applicable), such as the institution conducting the review.
    \end{itemize}

\item {\bf Declaration of LLM usage}
    \item[] Question: Does the paper describe the usage of LLMs if it is an important, original, or non-standard component of the core methods in this research? Note that if the LLM is used only for writing, editing, or formatting purposes and does not impact the core methodology, scientific rigorousness, or originality of the research, declaration is not required.
    \item[] Answer: \answerNA{}
    \item[] Justification: The core method development in this research does not involve LLMs as any important, original, or non-standard components.
    \item[] Guidelines:
    \begin{itemize}
        \item The answer NA means that the core method development in this research does not involve LLMs as any important, original, or non-standard components.
        \item Please refer to our LLM policy (\url{https://neurips.cc/Conferences/2025/LLM}) for what should or should not be described.
    \end{itemize}

\end{enumerate}

\end{document}